%% file: main.tex
\documentclass{article}

\usepackage[preprint]{neurips_2025}

\usepackage{natbib}
\usepackage[utf8]{inputenc} 
\usepackage[T1]{fontenc}    
\usepackage{hyperref}       
\usepackage{url}            
\usepackage{booktabs}       
\usepackage{amsfonts}       
\usepackage{nicefrac}       
\usepackage{microtype}      
\usepackage{xcolor}         

\usepackage{caption}
\usepackage{wrapfig}
\usepackage{subcaption}

\input{myheaders}

\usepackage{amsmath}
\usepackage{amssymb}
\usepackage{diagbox}
\usepackage{graphicx}
\usepackage{enumitem}

\title{Transformers as Multi-task Learners:\\ Decoupling Features in Hidden Markov Models}
\author{%
  Yifan Hao\thanks{Random order, equal contribution.} \quad \quad Chenlu Ye$^*$ \quad \quad  Chi Han \quad \quad Tong Zhang \\
  University of Illinois Urbana-Champaign \\
  \texttt{\{yifanh12, chenluy3, chihan3, tozhang\}@illinois.edu} \\
}

\begin{document}
\maketitle

\begin{abstract}
Transformer-based models have shown remarkable capabilities in sequence learning across a wide range of tasks, often performing well on specific task by leveraging input-output examples. Despite their empirical success, a comprehensive theoretical understanding of this phenomenon remains limited. In this work, we investigate the layerwise behavior of Transformers to uncover the mechanisms underlying their multi-task generalization ability. Taking explorations on a typical sequence model—Hidden Markov Models (HMMs), which are fundamental to many language tasks, we observe that:
(i) lower layers of Transformers focus on extracting feature representations, primarily influenced by neighboring tokens;
(ii) on the upper layers, features become decoupled, exhibiting a high degree of time disentanglement.
Building on these empirical insights, we provide theoretical analysis for the expressiveness power of Transformers. Our explicit constructions align closely with empirical observations, providing theoretical support for the Transformer’s effectiveness and efficiency on sequence learning across diverse tasks.
\end{abstract}

\section{Introduction}

Transformer-based models have achieved state-of-the-art performance across a broad range of sequence learning tasks, from language modeling and translation \citep{touvron2023llama, dubey2024llama, achiam2023gpt, team2023gemini} to algorithmic reasoning \citep{liu2024deepseek, ye2024physics}. Remarkably, a single Transformer can often generalize across diverse tasks with minimal supervision, leveraging only a few input-output examples—a capability that underpins its success in few-shot and in-context learning \citep{brown2020language, wei2022emergent, dong2022survey, min2022rethinking}.


While the empirical success is well-documented, a theoretical understanding of \emph{why Transformers generalize so effectively across tasks} remains elusive. In particular, the internal mechanisms by which Transformers represent and process sequential information across layers are not yet fully understood. This gap is especially pressing given the growing interest in deploying large-scale Transformers in multi-task and general-purpose settings.

In this work, we aim to bridge this understanding gap by investigating the layerwise behavior of Transformers. We take explorations on Hidden Markov Model (HMMs)\citep{rabiner1989tutorial,baum1967inequality}, a classical class of sequence models where observations depend on unobserved hidden states evolving underlying Markov dynamics.
Through empirical analysis, we uncover that while achieving good performance, Transformer learns feature representations on the lower layers, which are heavily influenced by nearby tokens, as well as developing decoupled features on upper layers, behaving like time disentangled representations (see Section~\ref{sec:empirical} for details). Motivated by these observations, we provide a theoretical analysis of Transformer expressiveness. By constructing explicit Transformer architectures that model HMMs efficiently, we demonstrate how the observed empirical patterns naturally emerge from our constructions. These results offer principled insights into how Transformers capture and generalize sequence information across tasks, shedding light on their success in multi-task and few-shot learning. The main contributions are summarized as follows:

\paragraph{Theoretical insights.} On the theoretical side, considering a large hidden space in practice, we adopt the low-rank structure for latent transitions in HMMs for further analysis, which has received tremendous attention recently for its efficiency in computation and inference \citep{siddiqi2010reduced,chiu2021low}. The theoretical results are stated as follows:
\begin{enumerate}[leftmargin=*]
\item \textbf{Expressiveness.} Under mild observability assumptions, Transformers can approximate low-rank HMMs using a fixed-length memory structure, enabling effective in-context learning.
\item \textbf{Layerwise modeling.} Our constructions mirror empirical observations: lower layers extract local features, which are then transformed into decoupled, task-relevant representations in upper layers.
\item \textbf{Generalization to ambiguous settings.} We extend our results to more challenging scenarios where the hidden state space exceeds the observation space, which are natural assumptions in NLP. And we show that Transformers can still learn expressive representations by composing features from multiple future observations.
\end{enumerate}

\paragraph{Empirical observations.} (See Section~\ref{sec:empirical} for details.)
\begin{enumerate}[leftmargin=*]
\item \textbf{Expressiveness.} Well-trained Transformers achieve high accuracy under in-context learning, with performance improving as more input-output examples are provided or as sequence length increases.
\item \textbf{Layerwise behavior.} Lower layers focus on learning local representations, primarily influenced by neighboring tokens. Upper layers develop decoupled, temporally disentangled representations that are less tied to specific input positions and encode higher-level abstractions.
\end{enumerate}

\subsection{Related works}

\begin{wrapfigure}[19]{t}{0.35\textwidth}
\vspace{-6mm}
    \centering
 \includegraphics[width=\linewidth]{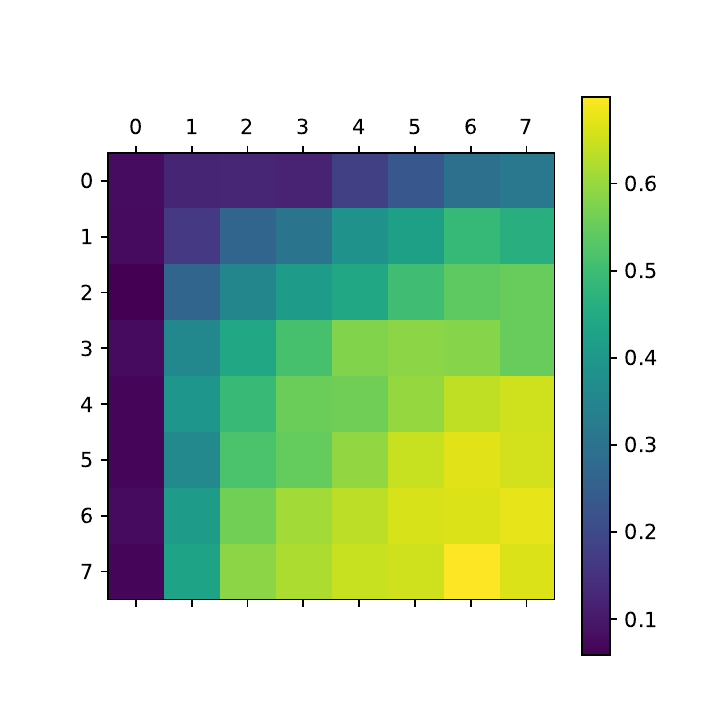}
    \caption{Accuracy of the Transformer under in-context learning setting. The y-axis denotes the number of demonstrative examples in-context, and the x-axis denotes the length of the test input $o_{test}$. All demonstrative examples have a length of 8 in this setting.}
    \label{fig:acc}
    \vspace{-6mm}
\end{wrapfigure}

The expressiveness of Transformers on sequence modeling has been explored from several perspectives. \citet{liu2022transformers} demonstrate that Transformers can emulate automata by learning deterministic transition patterns. \citet{nichani2024transformers} analyze a simplified setting where the data follows a Markov chain governed by a transition matrix. Other works, such as \citet{sander2024transformers} and \citet{wu2025transformers}, study the expressiveness of Transformers in autoregressive modeling, focusing on non-causal prediction tasks. In contrast, our work takes a first step toward understanding the expressive power of Transformers on Hidden Markov Models, 
which are arguably among the simplest yet fundamental tools for modeling natural language tasks. 

\section{Starting from the empirical findings}\label{sec:empirical}

\subsection{Experiment settings}

To empirically investigate how Transformers learn multiple tasks on sequential data, we construct a dataset generated by a mixture of Hidden Markov Models. Each HMM is used to model a tasks-specific distribution, and by mixing them we get a dataset similar to a pre-training corpus to learn language modeling.\footnote{More details are shown in Appendix~\ref{app:experiment}.} 
We sample 131k data, which allows training for 64 epochs, with 64 steps in each epoch on a batch size of 32.
We build a transformer of 16 layers and 16 heads in each layer, and a hidden state dimension of 1024. The transformer adopts the design of Roformer~\cite{su2024roformer} which uses rotary positional encoding technique, which determines the attention logit between two tokens based on their relative position.

\subsection{Results}

\paragraph{Expressiveness power on HMMs.} The high accuracy observed in Figure~\ref{fig:acc} highlights the expressiveness of well-trained Transformers. Moreover, we find that (1) accuracy improves as the number of input-output examples increases, and (2) task outputs become more predictable with longer test sequences.

\begin{wrapfigure}{t}{0.35\textwidth}
\vspace{-6mm}
    \centering
 \includegraphics[width=\linewidth]{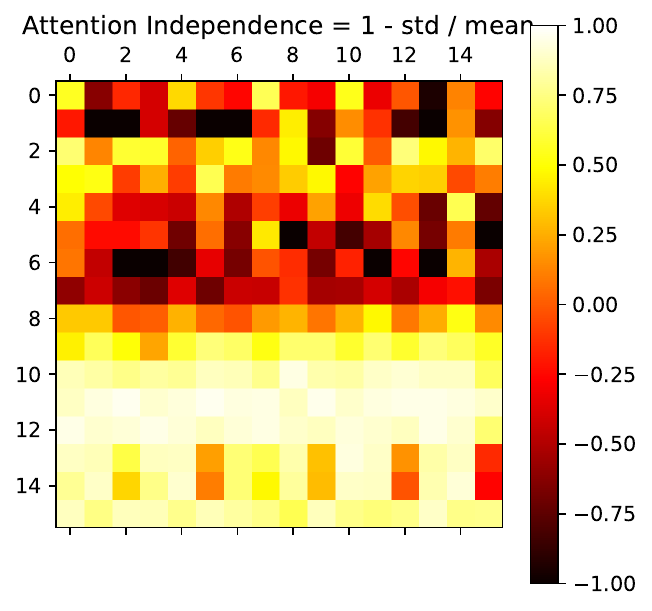}
    \caption{After randomly shuffling the positions of demonstrative inputs, we examine how the logits receive changes over layers (y-axis) and attention heads (x-axis). The measure is $1-\frac{\text{std}(\text{logits})}{\text{mean}(\text{logits})}$.}
    \label{fig:order}
    \vspace{-6mm}
\end{wrapfigure}

\paragraph{Decoupled features on upper layers.} We randomly shuffle the positions of demonstrative inputs and measure how the logit changes. As shown in Figure~\ref{fig:order}, the upper layers (layers 9–15) exhibit attention logits that are less dependent on the positions of input tokens. This suggests that feature representations in these layers become increasingly decoupled, reflecting a high degree of time disentanglement.

\paragraph{Layerwise investigations on Transformer recognitions.} Figure~\ref{fig:task_state} shows that Transformers gradually recognize the task identity across layers. Within a single task, the hidden state is identified earlier than the task itself, indicating that Transformers first learn the relationship between observations and hidden states in the lower layers, and then capture task-level structural information in the upper layers. This reflects a layerwise processing hierarchy in how Transformers handle sequential information.
In Figure~\ref{fig:token}, we observe three key patterns:
(1) The Transformer identifies previous tokens ($i{-}1$, $i{-}2$, $i{-}3$, $i{-}5$, $i{-}10$) with decreasing accuracy as the distance increases, suggesting that feature learning in lower layers relies primarily on nearby tokens.
(2) The accuracy curves for all distances follow a rising-then-falling trend across layers, implying that Transformers initially aggregate information from local contexts, and the resulting features then act as decoupled representations in upper layers.
(3) The peak of each curve shifts to upper layers as the distance to the previous token increases, showing that Transformers first integrate information from close neighbors and then progressively attend to more distant tokens.

\begin{figure}[htbp]
    \centering
    \begin{subfigure}[b]{0.4\textwidth}
        \centering
        \includegraphics[width=\textwidth]{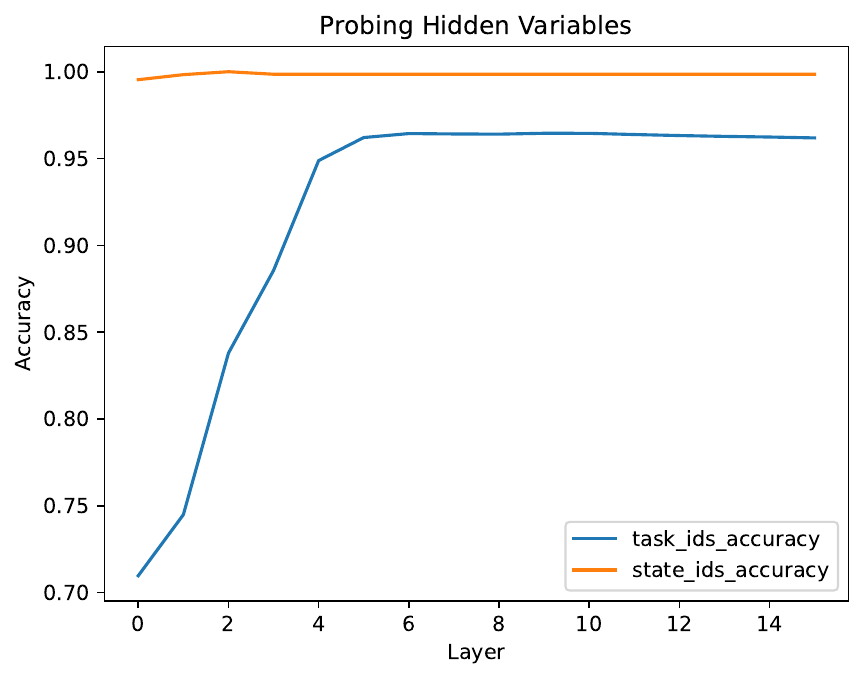}
        \caption{Task/ State recognition performance.}
        \label{fig:task_state}        
    \end{subfigure}
    \hspace{0.02\textwidth}
    \begin{subfigure}[b]{0.4\textwidth}
        \centering
        \includegraphics[width=\textwidth]{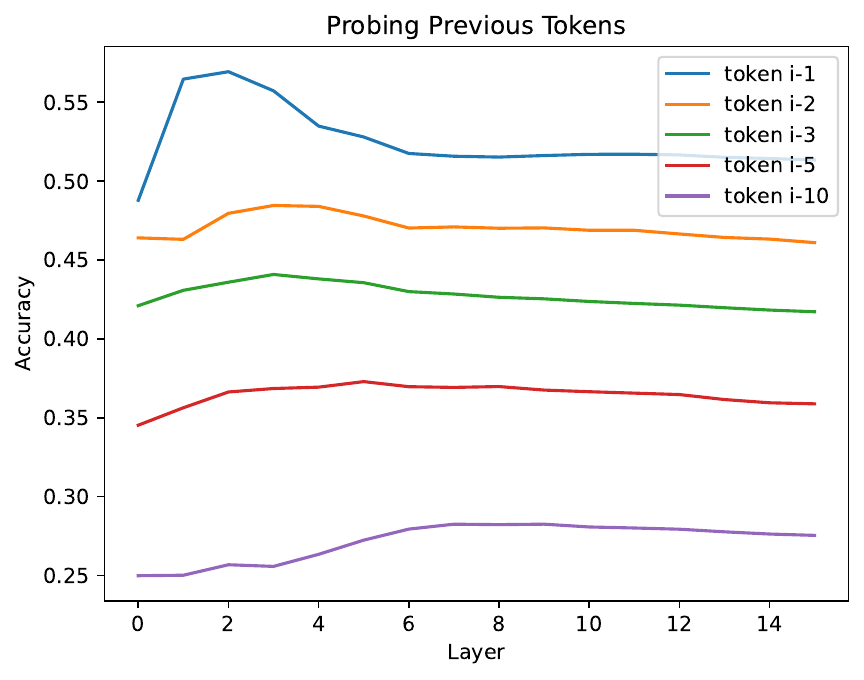}
        \caption{Token recognition performance.}
        \label{fig:token}
    \end{subfigure}
    \caption{Investigation on Transformer recognitions.}
\end{figure}

\section{Problem setup}

\subsection{Transformer Architecture}

We begin by describing the framework of Transformers as follows:

\paragraph{Attention head.} We first recall the definition of the (self-)Attention head $\attn(\cdot, Q, K, V)$. With any input matrix $M$, 
\begin{equation*}
    \attn \left( M, Q, K, V \right) = \sigma \left( M QK^T M^T \right) M V,
\end{equation*}
where $\{Q, K, V\}$ refer to the Query, Key and Value matrix respectively. The activation function $\sigma (\cdot)$ can be row-wise softmax function\footnote{Given a vector input $v$, the $i$-th element of $\softmax(v)$ is given by $\exp(v_i) / \sum_j \exp(v_j)$.} or element-wise ReLU function\footnote{$\relu(x) = \max\{ x, 0\}$}. 

\paragraph{Transformer.} Based on the architecture of Attention head, with the input matrix $M$, the definition of multi-head multi-layer Transformer $\mathrm{TF}(\cdot)$ is give by
\begin{align*}
  &  H^{(0)} = M, \quad H^{(l)} = H^{(l-1)} + \sum_{m=1}^{M_l} \attn \left( H^{(l-1)}, Q_m, K_m, V_m \right),
\end{align*}
for any $l \in [N]$, where $N$ refers to the number of Transformer layers, and $M_l$ is the number of Attention heads on the $l$-th layer.

\paragraph{One-hot encoding.} Considering a vector set with finite elements $\mathcal{S} := \{ v_1, v_2, \dots, v_m\}$, the One-hot encoding refers to mapping these vectors into $\mathrm{R}^m$, i.e, $\vecc(\cdot) : \mathcal{S} \to \mathrm{R}^m$. Each vector is mapped to an one-hot vector within $\{ e_1, e_2, \dots, e_m \}$, and for any two different vectors $v_{s}, v_{s'} \in \mathcal{S}$, there will be $\vecc(v_s) \ne \vecc (v_{s'})$.

\subsection{In-context Learning for Hidden Markov Model}
To show the expressive power of Transformers on sequence tasks, we consider a finite state case in this work, hidden Markov models (HMMs). To perform in-context learning, we collect $n$ i.i.d. demonstrate short observation sequences, i.e, $\{ o_{i,1}, \dots, o_{i,L} \}_{i=1}^n$, each sequence consists of $L-1$ observations. Denote the hidden state for each observation as $h_{i,s}$ for any $i \in [n], s \in [L]$, the HMM is defined as (more intuitive description is shown in Figure~\ref{fig:hmm}):
\begin{align*}
& \mathrm{P}(o_{i,s} | o_{i,1}, \dots, o_{i,s-1}, h_{i,1}, \dots, h_{i,s-1}, h_{i,s} ) = \mathrm{P}(o_{i,s} | h_{i,s}), \quad \forall i \in [n],  s \in [L],\\
& \mathrm{P}(h_{i,s} | o_{i,1}, \dots, o_{i,s-1}, h_{i,1}, \dots, h_{i,s-1}) = \mathrm{P}(h_{i,s} | h_{i,s-1}), \quad \forall i \in [n], s \in [L]. 
\end{align*}

\begin{figure}[t]
    \centering
    \includegraphics[width=0.5\linewidth]{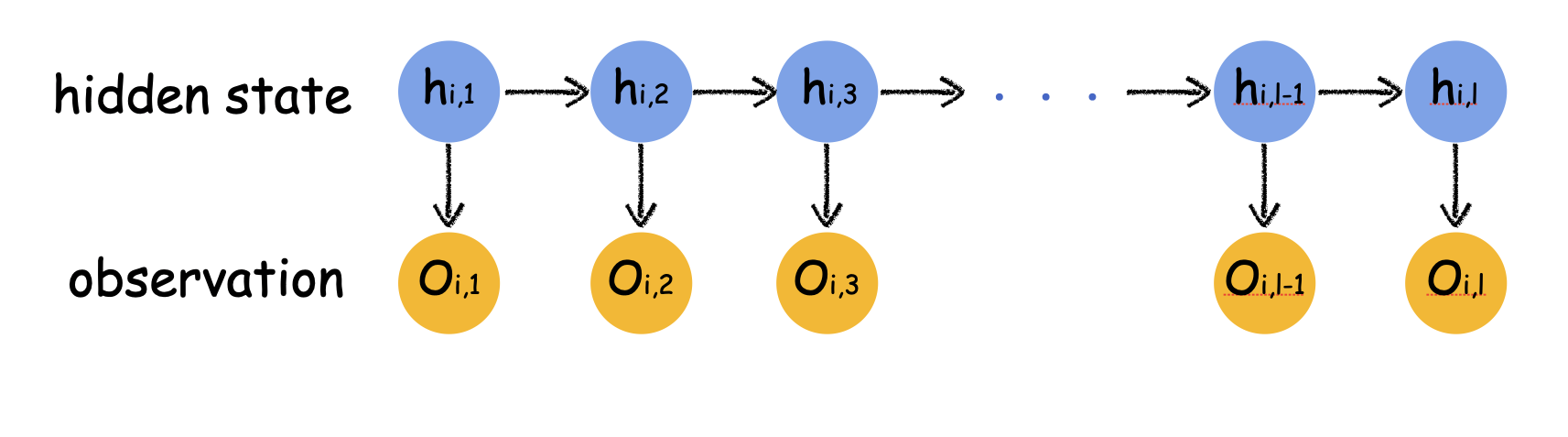}
    \caption{Illustration of Hidden Markov Model.}
    \label{fig:hmm}
\end{figure}

During testing, to predict $o_{\mathrm{test},k}$ given a long sequence history $\{ o_{\mathrm{test}, s} \}_{s=1}^{k-1}$, where $k > L$, we construct the input matrix $M_0$ for Transformers in the following format:
\begin{align*}
    M_0 := \begin{bmatrix}
        M_{0,1} & M_{0,2} & \cdots & M_{0,n} & M_{0,\te} \\
    \end{bmatrix}^T \in \mathrm{R}^{(n(L+1)+k)\times D},
\end{align*}
in which the column number $D$ will be specified later, and
\begin{align*}
   & M_{0,i} := \begin{bmatrix}
        o_{i,1} & o_{i,2} & \cdots & o_{i,L} & \od \\
        s_{(i-1)(L+1)+1} & s_{(i-1)(L+1) + 2} & \cdots & s_{i(L+1) -1} & s_{i(L+1)} \\
        v_{(i-1)(L+1)+1} & v_{(i-1)(L+1) + 2} & \cdots & v_{i(L+1) -1} & v_{i(L+1)} \\
    \end{bmatrix} \in \mathrm{R}^{D \times (L+1)}, \quad \forall i \in [n],\\
  &  M_{0,\te} := \begin{bmatrix}
          o_{\te,1} & o_{\te,2} & \cdots & o_{\te, k-1} & 0 \\
         s_{n(L+1)+1} & s_{n(L+1)+2} & \cdots & s_{n(L+1)+k-1} & s_{n(L+1)+k} \\
         v_{n(L+1)+1} & v_{n(L+1)+2} & \cdots & v_{n(L+1)+k-1} & v_{n(L+1)+k} \\
    \end{bmatrix} \in \mathrm{R}^{D \times k},
\end{align*}
where each column of $M_0$, i.e, $[o^T, s^T, v^T]$ represents the embedding for one observation, and $\od$ is the delimiter embedding, which represents the end of one sequence. The first $p+1$ dimension, i.e, $o$, refers to the token embedding, which is a one-hot vector within $\{ e_1, \dots, e_{p+1} \}$. Specifically,  we have
\begin{equation*}
    o \in \{ e_1, \dots, e_p \} \quad\text{for}~ o \ne \od, \quad \od = e_{p+1}.
\end{equation*}
The following two-dimensional vector $s$ is position embedding, which is referred to as
\begin{equation*}
   [ \pe_{pos} ]_{1} = \sin\left( \frac{pos}{1000 nk} \right), \quad  [ \pe_{pos} ]_{2} = \cos\left( \frac{pos}{1000 nk} \right) ,\quad \forall  1 \le pos \le n(L+1) + k.
\end{equation*}
And the last $(D- p - 3)$-dim vector $v$ is the fixed embedding, with elements of ones, zeros and indicators for being the test sequence:
\begin{equation*}
    v_{pos} := \big[\bm{0}_{D - p - 5}^\top, 1, \bm{1}(pos > n(L+1))^\top\big]^\top, \quad \forall 1 \le pos \le n(L+1) + k.
\end{equation*}
We will choose $D \ge 2 p^2 L$ to allocate sufficient capacity for storing the learned features. After feeding $M_0$ into the Transformer, we will obtain the output $\mathrm{TF}(M_0) \in \mathrm{R}^{(n(L+1)+k) \times D}$ with the same shape as the input, and \emph{read out} the conditional probability $\Pb(o_{\te,k} | o_{\te,1:k-1})$ from $[\mathrm{TF}(M_0)]_{(n(L+1)+k,1:p)}$ :
\begin{equation*}
    \hat{\Pb}(o_{\te, k} | o_{\te,1:k-1}) = \mathrm{read}(\mathrm{TF}(M_0)) := [\mathrm{TF}(M_0)]_{(n(L+1)+k,1:p)}.
\end{equation*}
The goal is to predict the conditional probability that is close to the true model.

\section{Theoretical analysis}\label{sec:theory}

\paragraph{Notation.} For a set $\cH$, we use $\Delta(\cH)$ to denote the set of all probability distributions on $\cH$. Let the emission operator $\mathbb T^*:\cH\rightarrow\Delta(\cO)$. For any $b\in\Delta(\cH)$, we use $\bbT^* b\in\Delta(\cO)$ to denote $\int_\cH\bbT^*(x|h)b(h)\md h$. For a vector $a$, we use $[a]_i$ to denote the $i$-th element of $a$. For a sequence $\{x_i\}_{i=1}^\infty$, we define the concatenated vector $x_{1:n}=[x_1,\ldots,x_n]^\top$. For a matrix $A\in\rR^{d_1\times d_2}$, we use $[A]_{(i,\cdot)}\in\rR^{d_2}$ and $[A]_{(\cdot,j)}\in\rR^{d_1}$ to denote the $i$-th row vector and the $j$-th column vector of $A$ respectively, use $[A]_{(i_1:i_2,\cdot)}$ and $[A]_{(\cdot, j_1:j_2)}$ to denote the submatrix 
consisting of rows $i_1$ through $i_2$, and the submatrix consisting columns $j_1$ through $j_2$ respectively. For a distribution $P:\{e_1,\ldots,e_p\}\rightarrow[0,1]$ supported on the tabular space, we define the vector $P(\cdot)=[P(e_1),\ldots,P(e_p)]^\top$\footnote{A more detailed notation table is provided in Table~\ref{tab:notation}. }.

\subsection{Low-rank HMM}\label{s:Low-rank HMM}
Our analysis is mainly based on the low-rank structure for HMM.

\begin{assumption}[Low rank structure]\label{as:low_rank}
 We suppose that the hidden state transition $\Pb:\cH\rightarrow\Delta(\cH)$ admits a low-rank structure: there exist two mappings $w^*,\psi^*:\cH\rightarrow\rR^d$ such that
\[
\Pb(h'|h) = w^*(h')^\top \psi^*(h).
\]
\end{assumption}
This condition requires that the latent transition has a low-rank structure, and the underlying representation maps $w^*,\psi^*$ are unknown. This structure is commonly used in representation learning \citep{agarwal2020flambe,uehara2021representation,uehara2022provably,guo2023provably}.

\begin{assumption}[Over-complete $\gamma$-Observability]\label{as:Over-complete_gam_obs}
    There exists $\gamma>0$ such that for any distributions $d,d'\in\Delta(\cH)$, we have
    \[
    \|\bbT d - \bbT d'\|_1 \ge \gamma\|d-d'\|_1.
    \]
\end{assumption}
This condition requires that the observation space is large enough to distinguish the hidden states by observations, i.e., the condition makes the reverse mapping from observation to hidden states a contraction. Observability is necessary and commonly assumed in HMM and partially observed systems \citep{uehara2022provably,guo2023provably}, and it is essentially equivalent to assuming that the emission matrix has full-column rank \citep{hsu2012spectral}. Further, Assumption \ref{as:Over-complete_gam_obs} implies that we can reverse the inequality to obtain the contraction from observation to hidden state distributions
$\|d-d'\|_1 \le \gamma^{-1} \|\bbT d - \bbT d'\|_1.$

Therefore, we can approximate the posterior hidden state distribution by a posterior sharing the same $(L-1)$-memory (refer to Lemma \ref{lm:Exponential Stability for Low-rank Transition}). Together with the low-rank condition that renders the transition $\Pb(o_k|o_{1:k-1}):= \mu^\top(o_k)\xi(o_{1:k-1})$, we can approximate $\Pb$ by a $(L-1)$-memory transition in the following lemma
\begin{equation*}
    \hat{\Pb}_L(o_k|o_{k-1:k-L-1}):=\mu(o_k)^\top \phi(o_{k-L+1:k-1}),
\end{equation*}
where $\mu(\cdot),\phi(\cdot)\in\rR^d$ denote the representations. The representation $\phi$ is a low-rank embedding of the belief distribution of hidden states. For simplicity, here we assume $\phi$ can be represented by a linear mapping.
\begin{lemma}[Model Approximation]\label{lem:model approximation}
    Under Assumptions \ref{as:low_rank} and \ref{as:Over-complete_gam_obs}, there exists a $(L-1)$-memory transition probability $\hat{\Pb}_L$ with $L=\Theta(\gamma^{-4}\log(d/\epsilon)$ such that
    \[
    \E_{o_{1:k-1}} \big\|\Pb(\cdot\mid o_{1:k-1}) - \Pb_L(\cdot\mid o_{k-L+1:k-1})\big\|_1 \le \epsilon.
    \]
\end{lemma}
This lemma shows that for a finite observability coefficient $\gamma$, the model approximation error can be controlled when the memory length $L-1$ is large enough. To prove this result, we bring the analysis techniques from POMDP literature \citet{guo2023transformers,uehara2022provably}. The detailed proof can be referred to Appendix \ref{ss: Proof for Model Approximation Error}.

\subsection{Main results}
\begin{assumption}\label{as:eigen_value}
    Given the data observation history $Z:=[o_{1,1:L-1},\ldots,o_{n,1:L-1}] \in \rR^{p(L-1) \times n}$, we suppose that the mean sample covariance $n^{-1}ZZ^\top$ has lower-bounded eigenvalue: $\lambda_{\min}(n^{-1}ZZ^\top) \ge \alpha.$
\end{assumption}
This assumption requires that the eigenvalues of the mean sample covariance are lower-bounded, implying that the data are distributed relatively evenly. This condition is commonly used in concentration analysis to bound the generalization error. Our main result can be formally stated as:
\begin{theorem}\label{thm:main}
Assume Assumption~\ref{as:low_rank}, \ref{as:Over-complete_gam_obs} and \ref{as:eigen_value} hold, there exists a $\cO(\ln L + T )$-layer Transformer $\mathrm{TF}_\theta$, such that for any input matrix $M_0$, with probability at least $1 - n^{-1}$ over $\{ o_{i,1}, \dots, o_{i,L} \}_{i=1}^n$:
\begin{equation*}
\begin{aligned}
    &\E_{ o_{\mathrm{test,1:k-1}}} \| \Pb(\cdot |  o_{\mathrm{test,1:k-1}} ) - \mathrm{read}\left( \mathrm{TF}_\theta(M_0) \right)\big] \|_1 \\
    &\le \underbrace{ \cO(d e^{-  \gamma^4 L})}_{\text{model approximation}} + \underbrace{\cO(p L^{1/2} e^{- \alpha T/(2L)}) }_{\text{optimization}} + \underbrace{ \cO(p L\sqrt{ \ln(nLp)}/(\sqrt{n}\alpha) + Ld/\alpha \cdot e^{-L\gamma^4})}_{\text{generalization}}.
\end{aligned}
\end{equation*}
\end{theorem}
The proof is in Appendix~\ref{pf:thm_main}. Theorem~\ref{thm:main} demonstrates that a sufficiently large Transformer can accurately approximate the HMM, revealing its strong expressive power in modeling sequential data.

\paragraph{Sources of errors.}
As shown in Lemma~\ref{lem:model approximation}, a fixed-length memory model is sufficient to approximate the full-memory transition probabilities, introducing only a small ``model approximation'' error. Our Transformer construction is based primarily on this approximation, denoted as $\Pb_L$. The ``generalization'' error arises due to the use of a finite sample size $n$: we learn $\Pb_L$ from $n$ i.i.d. samples, and the optimal learned model we can obtain, $\hat{\Pb}_L$, remains close to $\Pb_L$ as long as $n$ is sufficiently large. The final source of error, the ``optimization'' error, stems from the finite capacity of the Transformer. Since we approximate $\hat{\Pb}_L$ using a Transformer with a limited number of layers, a gap between the two remains. However, this gap can be made arbitrarily small by increasing the model size (e.g., number of layers), thereby improving the approximation accuracy.

\paragraph{Layerwise modeling.} Our explicit construction aligns closely with the empirical observations presented in Section~\ref{sec:empirical}. The construction proceeds in several stages. First, in the lower layers, the Transformer learns information from the neighborhood $L$ tokens, gradually incorporating information from nearby to more distant tokens, which is consistent with the patterns shown in Figure~\ref{fig:token}. In the upper layers, to take the final prediction, the learned features become decoupled and are used to infer a causal structure aligned with the underlying HMM task, which corresponds to Figure~\ref{fig:order} and the rising-then-falling trend observed in Figure~\ref{fig:token}. Finally, the overall progression—from token-level feature learning to task-level abstraction—matches the trends in Figure~\ref{fig:task_state}, reflecting a clear layerwise hierarchy in how Transformers process sequential information.

\subsection{Extension to indistinguishable situation}
In NLP tasks, a natural assumption is that the cardinality of hidden state space may be larger that the observation space evidence, or the true number of observations that can reveal the hidden states is small, called ``weak revealing" cases. In this section, we show that Transformer can still perform well under such ambiguous setting.
Inspired by the overcomplete POMDPs \citep{liu2022partially}, we start by expanding the output space of emission operators.

\begin{assumption}[Under-complete $\gamma$-Observability]\label{as:under-complete_gam_obs}
    Let operator $\bbM:\Delta(\cH)\rightarrow \Delta_m(\cO\times\cdots\times\cO)$ such that $\bbM d_\cH:\cO\times\cdots\times\cO\rightarrow\rR$ denotes $\int_{\cO\times\cdots\times\cO} \bbM(o_{t:t+m}|h_t) d_\cH(h_t) \mathrm{d}h_t$, where $m$ is a small constant such that $m < L$.
    There exists $\tilde{\gamma}>0$ such that for any distributions $d,d'\in\Delta(\cH)$, we have
    \[
    \|\bbM b - \bbM b'\|_1 \ge \tilde \gamma\|b - b'\|_1.
    \]
\end{assumption}

Then the corresponding theorem should be:\footnote{The conditional probability in Theorem~\ref{thm:extend} is related to a $m$-step prediction, which induces that the cardinality of observation is $p^m$. So we enlarge $D$ such that $D \ge 2 p^m L$, and the read out function should be $ \hat{\Pb}(o_{\te, k} | o_{\te,1:k-1}) = \mathrm{read}(\mathrm{TF}(M_0)) := [\mathrm{TF}(M_0)]_{(n(L+1)+k,(L+1)(p+3)+1 : (L+1)(p+3)+p^m)}$. }

\begin{theorem}\label{thm:extend}
Denote the data observation $Z' := [o_{1,1:L-m}, \dots, o_{n,1:L-m}] \in \rR^{p(L-m) \times n}$.
Assume Assumption~\ref{as:low_rank}, \ref{as:under-complete_gam_obs} hold, and $\lambda_{\min} (n^{-1} Z' Z^{' T}) \ge \alpha$,
there exists a $\cO(\ln L  + T)$-layer Transformer $\mathrm{TF}_\theta$, such that for any input matrix $M_0$, with probability at least $1 - n^{-1}$ over $\{ o_{i,1}, \dots, o_{i,L} \}_{i=1}^n$:
\begin{equation*}
\begin{aligned}
    &\E_{ o_{\mathrm{test,1:k-1}}} \| \Pb(\cdot |  o_{\mathrm{test,1:k-1}} ) - \mathrm{read}\left( \mathrm{TF}_\theta(M_0) \right) \|_1\\
    &\quad\le \underbrace{ \cO(d e^{-  \tilde{\gamma}^4 L})}_{\text{model approximation}} + \underbrace{\cO(p^m L^{1/2} e^{- \alpha T/(2L)}) }_{\text{optimization}} + \underbrace{ \cO(p^m L\sqrt{ \ln(nLp)}/(\sqrt{n}\alpha) + Ld/\alpha \cdot e^{-L\tilde{\gamma}^4})}_{\text{generalization}}.
\end{aligned}
\end{equation*}
\end{theorem}
The proof is in Appendix~\ref{pf:thm_extend}. From Theorem~\ref{thm:extend}, we show that Transformers can still learn HMMs efficiently under such ``weak revealing'' case, by concatenating several steps of future observations.
\begin{figure}[t]
    \centering
    \includegraphics[width=0.5\linewidth]{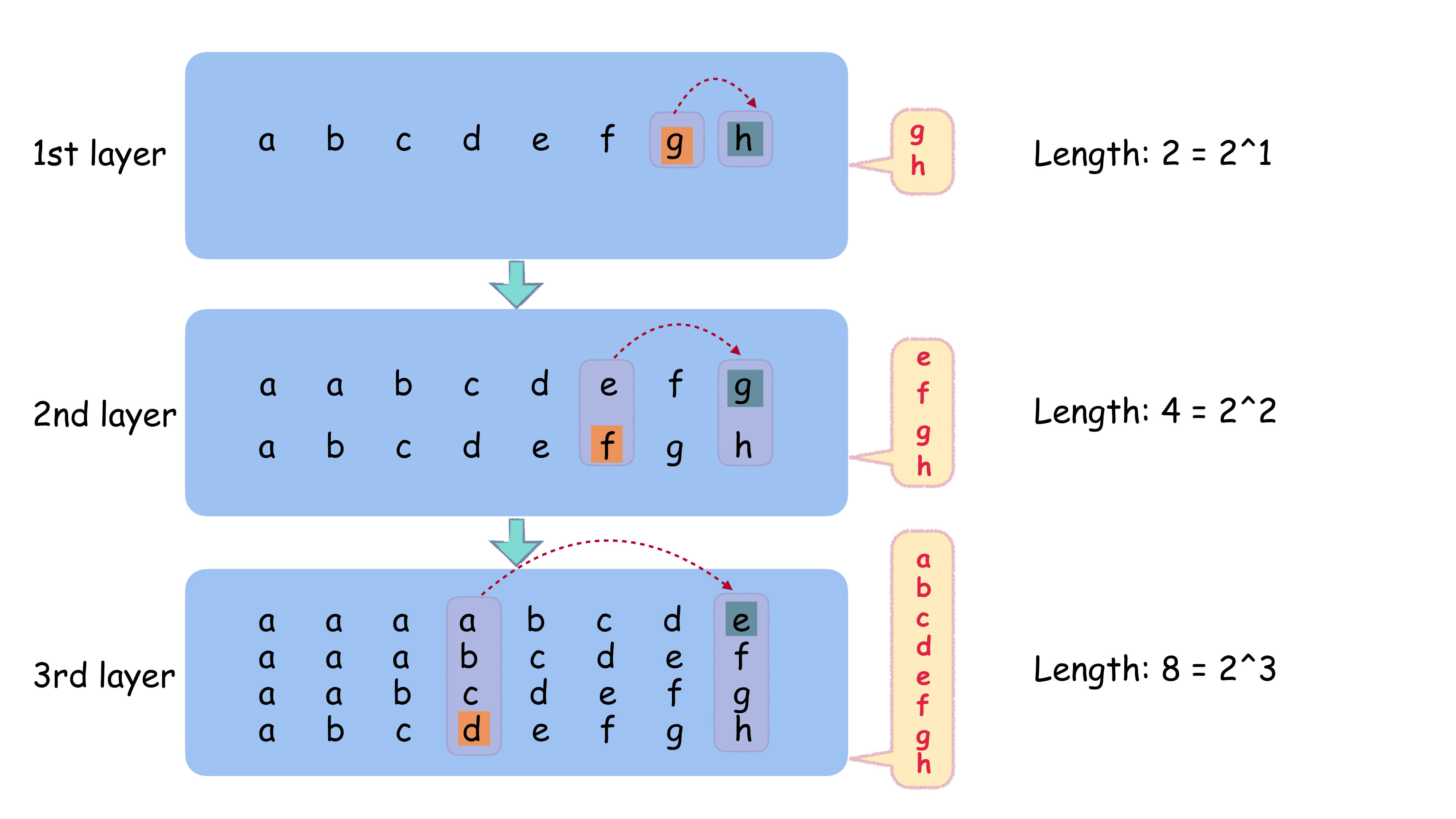}
    \caption{Illustration of Feature learning process.}
    \label{fig:feature}
\end{figure}

\section{Transformer Construction and Proof Sketches}
\subsection{Proof Sketches for Theorem~\ref{thm:main}}
Recalling Lemma~\ref{lem:model approximation}, our Transformer construction is mainly based on approximating $\Pb_L(\cdot | o_{\te, k-L+1:k-1})$ with expression: $\Pb_L(o_k | o_{k-L+1:k-1}) = \mu^\top(o_k) \phi(o_{k-L+1:k-1}).$

To approximate the error in prediction, we can take the following decomposition:
\begin{equation}\label{eq:decompose_main}
\footnotesize
\begin{aligned}
& \quad \E_{o_{\te, 1:k-1}} \| \Pb (\cdot | o_{\te,1:k-1}) - \mathrm{read}(\mathrm{TF}_\theta(M_0)) \|_1 \\
&\le \underbrace{ \E_{o_{\te, 1:k-1}} \| \Pb (\cdot | o_{\te,1:k-1}) - \Pb_L(\cdot| o_{\te,k-L+1:k-1}) \|_1}_{\epsilon_1:\text{model approximation}} \\
& + \underbrace{\E_{o_{\te, 1:k-1}} \|  \Pb_L(\cdot| o_{\te,k-L+1:k-1}) - \hat{\Pb}_L(\cdot| o_{\te,k-L+1:k-1})\|_1}_{\epsilon_2:\text{generalization}} \\
& + \underbrace{ \E_{o_{\te, 1:k-1}} \|\hat{\Pb}_L(\cdot| o_{\te,k-L+1:k-1})  - \mathrm{read}(\mathrm{TF}_\theta(M_0)) \|_1 }_{\epsilon_3:\mathrm{optimization}},
\end{aligned}
\end{equation}
where $\hat{\Pb}_L(\cdot| o_{\te,k-L+1:k-1})\in\rR^p$ refers to the optimal approximation for $\Pb_L$ based on $n$ i.i.d. samples we collected.
Considering the one-hot format of $o_k$ and the linear assumption on $\phi(\cdot)$, we can express both $\mu(\cdot)$ and $\phi(\cdot)$ as linear function, which implies that 
\begin{equation*}
 \Pb_L(\cdot | o_{k-L+1:k-1})  := W_* o_{k-L+1:k-1},
\end{equation*}
for some $W_* \in \rR^{p \times p(L-1)}$\footnote{As the $(p+1)$-th dimension is designed only for $\od$, we consider the observation as a $p$-dim vector for simplicity.}. Accordingly, we have $\hat{\Pb}_L(\cdot| o_{\te,k-L+1:k-1}):=\hat{W} o_{\te,k-L+1:k-1}$, in which
\begin{equation}\label{eq:prob_vec_n_m}
\begin{aligned}
    \hat{W} :=  \arg \min_W \mathcal{L}(W) := \arg \min_W \sum_i \| o_{i,L} - W z_i \|_2^2. 
\end{aligned}    
\end{equation}
Here we use the short-hand notation $z_i:=o_{i,1:L-1} \in \mathrm{R}^{p(L-1)}$. 
From Lemma \ref{lem:model approximation}, we obtain $\epsilon_1 = \cO(d e^{- \gamma^4 L})$. And in the following analysis, we focus on bounding $\epsilon_2$ and $\epsilon_3$, respectively.

\subsubsection{Transformer Construction}
To predict the conditional probability vector $\hat{\Pb}_L(\cdot |  o_{\mathrm{test},k-L+1:k-1})$, the transformer proceeds in three main steps: (i) it first learns the $(L-1)$-step history feature $o_{i,1:L-1}$ associated with $o_{i,L}$, as well as $o_{\te,k-L+1:k-1}$ associated with $o_{\te,k}$, (ii)it then performs linear regression based on Eq.~\eqref{eq:prob_vec_n}, (iii)finally, it approximates $\hat{\Pb}_L(\cdot |  o_{\mathrm{test},k-L+1:k-1})$ using $\hat{W}$ and $o_{\te, k-L+1:k-1}$.
The explicit construction of the Transformer is detailed below:
\paragraph{Decoupled feature learning.} Before formally construction, for any step index $1 \le r < L$, we define history and future matrix $Z_r, F_r \in \rR^{(n(L+1)+k) \times (p+3)}$ for further analysis:
\begin{equation*}
\footnotesize
\begin{aligned}
    [Z_r]_{(t,\cdot)} :=& \left\{ \begin{aligned}
        & [M_0]_{(t-r,1:p+3)}, \quad r < t \le n(L+1)+k,\\
        & [M_0]_{(1,1:p+3)}, \quad 1 \le t \le r,
    \end{aligned}\right.\\
    [F_r]_{(t,\cdot)} :=& \left\{ \begin{aligned}
        & [M_0]_{(t+r,1:p+3)}, \quad 1 \le t \le n(L+1)+k-r,\\
        & [M_0]_{(n(L+1)+k,1:p+3)}, \quad n(L+1)+k-r < t \le n(L+1)+k.
    \end{aligned}\right.
\end{aligned}
\end{equation*}
To be specific, for each $o_{i,s}$, $Z_r$ and $F_r$ are corresponding to $o_{i,s-r}$ (history observation) and $o_{i,s+r}$ (future observation) respectively. To learn these two types of features, we use two special matrices on the position embedding vector of each observation:
\begin{equation*}
  A := \beta_1  \begin{bmatrix}
      &  \cos(\frac{1}{1000nk}) & \sin(\frac{1}{1000nk}) \\
      & - \sin(\frac{1}{1000nk}) & \cos(\frac{1}{1000nk}) \\
    \end{bmatrix}, \quad  B := \beta_1  \begin{bmatrix}
      &  \cos(\frac{1}{1000nk}) & -\sin(\frac{1}{1000nk}) \\
      &  \sin(\frac{1}{1000nk}) & \cos(\frac{1}{1000nk}) \\
    \end{bmatrix}.
\end{equation*}
For $t_1,t_2\in[1:n(L+1)+k]$ with position embedding vectors $s_{t_1}, s_{t_2}$, we have
\begin{equation*}
    s_{t_1}^T A s_{t_2} = \beta_1 \cdot \cos \left( \frac{t_1 - t_2 -1}{1000nk} \right), \quad s_{t_1}^T B s_{t_2} = \beta_1 \cdot \cos \left( \frac{t_1 - t_2 +1}{1000nk} \right).
\end{equation*}
\begin{figure}[t]
    \centering
    \includegraphics[width=0.6\linewidth]{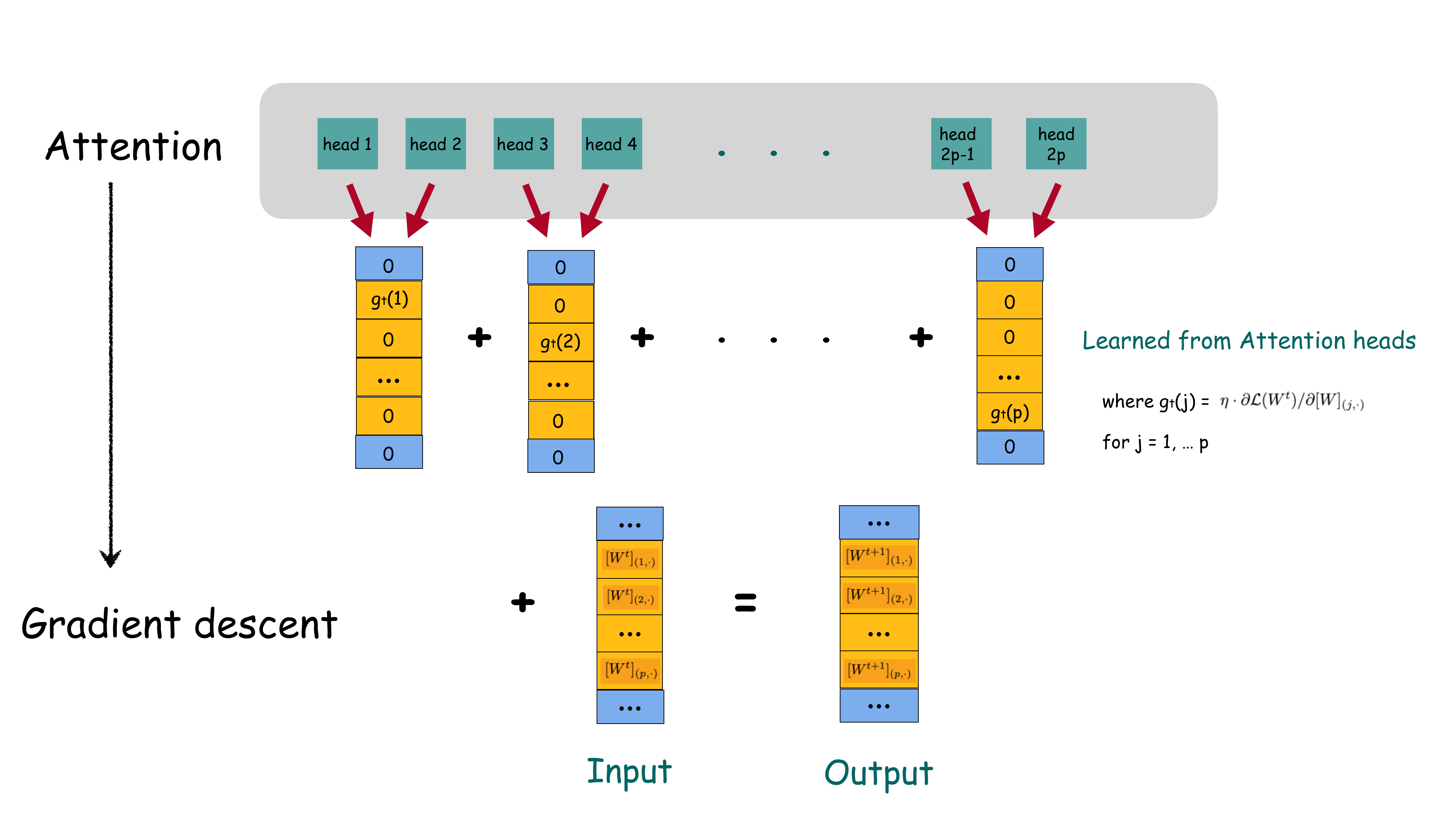}
    \caption{Illustration of gradient descent performance.}
    \label{fig:gd}
\end{figure}
By using A in Query-Key matrix with enough large $\beta_1$, and applying the softmax activation along with a carefully designed Value matrix, we can learn $Z_1$ after the first Attention layer. On the second layer, we again use $A$ to design Query-Key matrix, which enables the learning of $Z_2, Z_3$ (see Figure~\ref{fig:feature} as a detailed illustration). Repeating such process for $\cO(\ln L)$ layers, we will obtain $\{ Z_1, \dots, Z_{L-1}\}$ using $\cO(\ln L)$-layer single-head Attention. Also, use matrix $B$, we can obtain $F_1$ on the following layer. The output matrix after these decoupled-feature layers should be
\begin{equation*}
    M_{\text{dec}} = [[M_0]_{(\cdot, p+3)}, Z_1, Z_2, Z_3, \dots, Z_{L-1}, F_1, [M_0]_{(\cdot, (L+1)(p+3)+1:D)} ].
\end{equation*}
\paragraph{Gradient descent performing and final prediction.} The following $\cO(T)$-layer architecture is designed to learn $\hat{\Pb}_L(\cdot | z_k)$ based on history information $\{Z_1, \dots, Z_{L-1}\}$. 
To be specific, from Eq.~\eqref{eq:prob_vec_n_m}, we need to take linear regression to estimate a matrix $\hat{W} \in \rR^{p \times p(L-1)}$.
To perform such estimation process for $W$, we construct a $2p$-head $\cO(T)$-layer Attention. Each layer can perform single gradient descent step on $\mathcal{L}(W)$, starting from an initial value $0$. Each row of $W$ is assigned to two independent attention heads for parallel learning (see Figure~\ref{fig:gd} for detailed illustration). The construction closely follows the method proposed in \citet{bai2024transformers}, with the key difference being that we use $F_1$ to pick up $n$ samples for the gradient descent updating. After $\cO(T)$-step gradient descent, we use the learned $\{[\hat W]_{(1,\cdot)}, \dots, [\hat W]_{(p,\cdot)} \}$ and $o_{\te,k-L+1:k-1}$ to predict $\hat{\Pb}_L(\cdot |  o_{\mathrm{test},k-L+1:k-1})$. The corresponding error $\epsilon_3 = \cO(p L^{1/2} e^{-\alpha T / (2L)})$ can be estimated using Lemma~\ref{lem:reg}.

\subsubsection{Generalization Error Approximation}
Using the notations for labels and covariates $O:=[o_{1,L}, \ldots, o_{n,L}] \in \rR^{p \times n},~ Z=[o_{1,1:L-1},\ldots,o_{n,1:L-1}] \in\rR^{p(L-1)\times n}$, the least square estimator has the following closed-form solution: $ \hat{W} := O Z^T (ZZ^T )^{-1} $.

Then, denoting $z_\te := o_{\te,k-L+1:k-1}$ and error $\Delta := O - W_* Z$, we can take the estimator into $\epsilon_2$ and upper bound it by
\begin{equation}\label{eq:eps_2}
\footnotesize
\begin{aligned}
    \epsilon_2\le& \sum_{j=1}^p \sqrt{L} \| [W_*]_{(j,\cdot)} - [O]_{(j,\cdot)} Z^T(ZZ^T)^{-1} \|_2
    \le \frac{\sqrt{L}}{n\alpha} \sum_{j=1}^p \| [\Delta]_{(j,\cdot)} Z^T\|_2\\
    \le& \frac{\sqrt{L}}{n\alpha} \sum_{j=1}^p \| \big([\Delta]_{(j,\cdot)} - \E[[\Delta]_{(j,\cdot)}] \big) Z^T \|_2 + \frac{\sqrt{L}}{n\alpha} \sum_{j=1}^p \| \E[[\Delta]_{(j,\cdot)}]\big) Z^T \|_2
\end{aligned}
\end{equation}
where the second inequality uses the definition $O=\Delta + W_* Z$ and $\lambda_{\min}(ZZ^\top)\ge \alpha$ in Assumption \ref{as:eigen_value}, and invokes the Cauchy-Schwartz inequality. For the first term on the last row of \eqref{eq:eps_2}, we use the matrix concentration in Lemma \ref{lem:concen} to obtain that with a high probability,
\begin{align*}
    \| \big([\Delta]_{(j,\cdot)} - \E[[\Delta]_{(j,\cdot)}] \big) Z^T \|_2 \le \cO\big(\sqrt{nL \ln(nLp^2)}\big).
\end{align*}
For the second term on the last row of \eqref{eq:eps_2}, based on the observation that $\E[[\Delta]_{(j,i)}]=\E_{o_{i,1:k-1}}[\Pb(e_j\mid o_{1:k-1}) - \Pb_L(e_j\mid o_{k-L+1:k-1})]$, we can bound it by $ \cO(Ld/\alpha \cdot e^{-L\gamma^4})$ via Lemma \ref{lem:model approximation}.

\subsection{Proof Sketches for Theorem~\ref{thm:extend}}
The error analysis and the corresponding Transformer construction follow a similar approach to Theorem~\ref{thm:extend}, with one key modification. After the decoupled feature extraction stage, the resulting output matrix takes the following form:
\begin{equation*}
    M_{\text{dec}} = [ [M_0]_{(\cdot, 1:p+3)}, Z_1, Z_2, \dots, Z_{L-m}, F_1, F_2, \dots, F_m, [M_0]_{(\cdot,(L+1)(p+3)+1 :D)}].
\end{equation*}
Before feeding it into subsequent Attention layers, we apply an one-hot encoding function $\vecc(\cdot)$ to each row of $\{ [M_0]_{(\cdot, 1:p)}, [F_1]_{(\cdot,1:p)}, \dots, [F_{m-1}]_{(\cdot,1:p)}\}$, which correspond to the current and future observations at each time step.

\section{Conclusion}\label{sec:conclu}
This work advances our theoretical and empirical understanding of how Transformers achieve strong generalization across diverse sequence learning tasks. By analyzing their layerwise behavior and constructing explicit architectures for modeling HMMs, we demonstrate that Transformers gradually transition from learning local, token-level features in lower layers to forming decoupled representations in upper layers. These findings not only align with empirical observations but also provide a principled explanation for the Transformer’s expressiveness and efficiency in multi-task and in-context learning settings. Our results contribute to a deeper understanding of how modern sequence models internalize structure and support general-purpose sequence processing. In future work, we aim to extend our analysis to more general assumptions and more complex model settings.

\bibliography{sample}
\bibliographystyle{apalike}

\clearpage
\newpage

\appendix

\section{Related works}

\paragraph{Expressiveness of Transformer.} The expressive power of Transformers has been studied extensively from various perspectives. For example, \citet{akyurek2022learning, von2023transformers, mahankali2023one, dai2022can} demonstrate that a single attention layer is sufficient to compute a single gradient descent step. \citet{garg2022can, bai2024transformers, guo2023transformers} show that Transformers can implement a wide range of machine learning algorithms in context. Similarly, \citet{xie2021explanation, wang2023large, jiang2023latent} establish that Transformers can approximate Bayesian optimal inference. Other works have explored different capabilities of Transformers: \citet{liu2022transformers} show they can learn shortcuts to automata, \citet{lin2023transformers} demonstrate their ability to implement reinforcement learning algorithms, and \citet{nichani2024transformers} reveal their capacity to learn Markov causal structures under a fixed transition matrix, \citet{sander2024transformers, wu2025transformers} show the expressiveness power on learning autoregressive models.

\paragraph{Hidden Markov Model.} Identification for uncontrolled partially observable systems has been broadly studied, especially for the spectral learning based models \citep{hsu2012spectral,van1995unifying,song2010hilbert,hamilton2013modelling,kulesza2015spectral}. Intuitively, all the frameworks require some observability conditions to reveal the hidden states via sufficient observations. For complex sequential spaces with a large hidden state space, there is another line of work considering structured latent transitions, allowing for more efficient inference and computation complexity \citep{siddiqi2005fast,felzenszwalb2003fast,dedieu2019learning,siddiqi2010reduced,chiu2021low}. Especially, \citet{chiu2021low} consider a low-rank structure for hidden state transitions. Such a low-rank structure is also widely studied in partially observable Markov Decision processes \citep{uehara2022provably,guo2023provably,zhong2022gec,wang2022embed,zhan2022pac}. The most related ones to our work are \citet{uehara2022provably,guo2023provably}, which utilize the low-rank latent transition and observability to avoid a long-memory learning and inference. Instead, they can approximate the posterior distribution of the hidden states given whole observations by a distribution conditioned on a fixed-size history.

\section{Additional experiment details and results}\label{app:experiment}

\subsection{Experiment settings}
Here we construct a dataset generated by a mixture of Hidden Markov Models (HMMs). Each HMM is used to model a tasks-specific distribution, and by mixing them we get a dataset similar to a pre-training corpus to learn language modeling on. In specific, we randomly simulate 8192 HMMs. The generation process is as follows. There is an initial task distribution on which we sample the HMM id. Each HMM composes of 128 hidden states randomly transiting between each other. Each next state depends purely on the previous state, making the sequence of hidden states Markovian. All HMMs share a 16-token vocabulary. Each hidden state is associated with an emission distribution to randomly output a token. We sample 131k data, which allows training for 64 epochs, with 64 steps in each epoch on a batch size of 32.
We build a transformer of 16 layers and 16 heads in each layer, and a hidden state dimension of 1024. The experiments run on a single V100 GPU with 16 GB of memory for 10 hours. The mixture-of-HMMs simulation runs with default multiprocessing of Python.

\subsection{Additional results}
See Figure~\ref{fig:hmm_heatmap} for the attention heatmap.

\begin{figure}[h]
    \centering
    \includegraphics[width=0.4\linewidth]{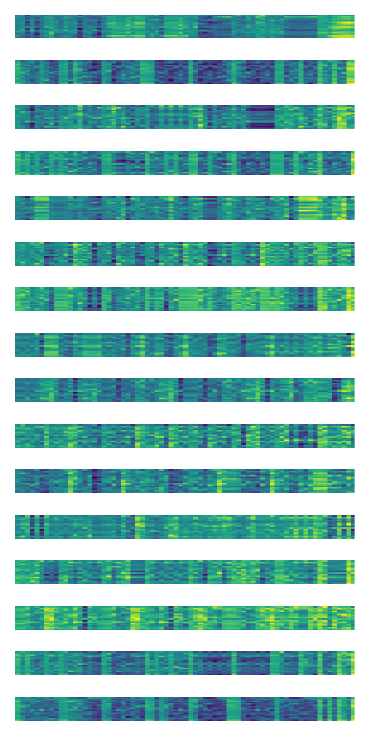}
    \caption{Attention of the Transformer on in-context learning inputs. The y-axis denotes layers and attention heads within each layers, and the x-axis denotes the attention of the last token on all previous tokens in the ICL input (including both demonstrative examples and the test input).}
    \label{fig:hmm_heatmap}
\end{figure}

\section{Notation Table}

\begin{table}[h]
\centering 
\caption{The table of notations used in this paper.}
\begin{tabular}{c|l}
\toprule\midrule
\textbf{Notation} & \textbf{Description} \\
\midrule
$\Delta(\cH)$ & the set of all probability distributions on $\cH$ \\
$\mathbb T^*$ & the emission operator \\
$\bbT^* b$ & $\int_\cH\bbT^*(x|h)b(h)\md h$ \\
$e_j$ & one-hot vector \\
$[a]_j$ & the $i$-th element of vector $a$ \\
$x_{1:n}$ & concatenated vector $[x_1,\ldots,x_n]^\top$ \\
$[A]_{(i,\cdot)}$ & the $i$-th row vector of $A$ \\
$[A]_{(\cdot,j)}$ & the $j$-th column vector of $A$ \\
$[A]_{(i_1:i_2,\cdot)}$ & the submatrix 
consisting of rows $i_1$ through $i_2$ of $A$ \\
$[A]_{(\cdot, j_1:j_2)}$ & the submatrix consisting columns $j_1$ through $j_2$ of $A$ \\
$P(\cdot)$ & the vector $[P(e_1),\ldots,P(e_p)]^\top$ for a distribution $P:\{e_1,\ldots,e_p\}\rightarrow[0,1]$ \\
$L$& sequence length on training samples  \\
$\gamma$ & observability coefficient\\
$p$ & observation state number \\
$d$ & feature dimension in transition matrix low-rank structure \\
$n$ &sequence sample number \\
$k$ & sequence length on test sample \\
$T$ & the number of gradient descent steps after feature obtaining \\

\bottomrule
\end{tabular}
\label{tab:notation}
\end{table}

\section{Proofs for Theorem \ref{thm:main}}\label{pf:thm_main}
Recalling Lemma~\ref{lem:model approximation}, our Transformer construction is mainly based on approximating $\Pb_L(\cdot | o_{\te, k-L+1:k-1})$ with expression:
\begin{equation*}
    \Pb_L(o_k | o_{k-L+1:k-1}) = \mu(o_k)^T \phi(o_{k-L+1:k-1}).
\end{equation*}
To approximate the error in prediction, we can take the following decomposition:
\begin{equation}\label{eq:decompose}
\begin{aligned}
&\E_{o_{\te, 1:k-1}} \| \Pb (\cdot | o_{\te,1:k-1}) - \mathrm{read}(\mathrm{TF}_\theta(M_0)) \|_1\\
&\le \underbrace{ \E_{o_{\te, 1:k-1}} \| \Pb (\cdot | o_{\te,1:k-1}) - \Pb_L(\cdot| o_{\te,k-L+1:k-1}) \|_1}_{\epsilon_1:\text{model approximation}} \\
& + \underbrace{\E_{o_{\te, 1:k-1}} \|  \Pb_L(\cdot| o_{\te,k-L+1:k-1}) - \hat{\Pb}_L(\cdot| o_{\te,k-L+1:k-1})\|_1}_{\epsilon_2:\text{generalization}} \\
& + \underbrace{ \E_{o_{\te, 1:k-1}} \|\hat{\Pb}_L(\cdot| o_{\te,k-L+1:k-1})  - \mathrm{read}(\mathrm{TF}_\theta(M_0)) \|_1 }_{\epsilon_3:\mathrm{optimization}},
\end{aligned}
\end{equation}
where $\hat{\Pb}_L(\cdot| o_{\te,k-L+1:k-1})\in\rR^p$ refers to the optimal approximation for $\Pb_L$ based on $n$ i.i.d. samples we collected.

Considering the one-hot vector $o_k \in \rR^p$, which representing the observation state\footnote{As the $(p+1)$-th dimension is designed only for $\od$, we consider the observation as a $p$-dim vector in proofs for simplicity.}, we can express $\mu(\cdot)$ as
\begin{equation*}
    \mu(o_k) = U o_k,
\end{equation*}
for some $U \in \rR^{d \times p}$. Also, recalling the linear mapping assumption for $\phi(\cdot)$, we can also obtain
\begin{equation*}
    \phi(o_{k-L+1:k-1}) = V o_{k-L+1:k-1},
\end{equation*}
for some $V \in \rR^{d \times p(L-1)}$,
which further implies that
\begin{equation*}
    \Pb_L(o_k|o_{k-L+1:k-1}) = o_{k}^T U^T V o_{k-L+1:k-1}.
\end{equation*}
As the feature embeddings are within $\{ e_1, \dots, e_p \}$, the vector $\Pb_L(\cdot | o_{k-L+1:k-1})\in \rR^p$ equals to
\begin{equation}\label{eq:prob_vec}
    \Pb_L(\cdot | o_{k-L+1:k-1}) = U^T V o_{k-L+1:k-1} := W_* o_{k-L+1:k-1},
\end{equation}
where $W_* \in \rR^{p \times p(L-1)}$.
So for $\hat{\Pb}_L(\cdot| o_{\te,k-L+1:k-1}):=\hat{W} o_{\te,k-L+1:k-1}$, the solution is
\begin{equation}\label{eq:prob_vec_n}
\begin{aligned}
    \hat{W} :=  \arg \min_W \mathcal{L}(W) := \arg \min_W \sum_i \| o_{i,L} - W z_i \|_2^2,  
\end{aligned}    
\end{equation}
where we use the short-hand notation $z_i:=o_{i,1:L-1} \in \mathrm{R}^{p(L-1)}$.
From Lemma \ref{lem:model approximation}, we have that $\epsilon_1 = \cO(d e^{- \gamma^4 L})$. In the following two subsections, we focus on bounding $\epsilon_2$ and $\epsilon_3$, respectively.

\subsection{Transformer Construction}

To approximate the conditional probability vector $\hat{\Pb}_L(\cdot |  o_{\mathrm{test},k-L+1:k-1})$, the transformer mainly takes three steps: (1) firstly learning the $(L-1)$-step history features $o_{i,1:L-1}$ for $o_{i,L}$, as well as $o_{\te,k-L+1:k-1}$ for $o_{\te,k}$, (2) then performing linear regression based on Eq.~\eqref{eq:prob_vec_n}, (3) finally approximating $\hat{\Pb}_L(\cdot |  o_{\mathrm{test},k-L+1:k-1})$ using $\hat{W}$ and $o_{\te, k-L+1:k-1}$.
The explicit construction of the Transformer is as follows:

\paragraph{Decoupled feature learning.} Here we first construct an $\cO(\ln L)$-layer single head Attention, to learn $o_{i,1:L-1}$ for $o_{i,L}$, as well as $o_{\te,k-L+1:k-1}$ for $o_{\te,k}$. Before formally construction, for any step index $1 \le r < L$, we define history and future matrix $Z_r, F_r \in \rR^{(n(L+1)+k) \times (p+3)}$ for further analysis:
\begin{equation*}
\begin{aligned}
    [Z_r]_{(t,\cdot)} :=& \left\{ \begin{aligned}
        & [M_0]_{(t-r,1:p+3)}, \quad r < t \le n(L+1)+k,\\
        & [M_0]_{(1,1:p+3)}, \quad 1 \le t \le r,
    \end{aligned}\right.\\
    [F_r]_{(t,\cdot)} :=& \left\{ \begin{aligned}
        & [M_0]_{(t+r,1:p+3)}, \quad 1 \le t \le n(L+1)+k-r,\\
        & [M_0]_{(n(L+1)+k,1:p+3)}, \quad n(L+1)+k-r < t \le n(L+1)+k,
    \end{aligned}\right..
\end{aligned}
\end{equation*}
Here we also define a special matrix 
\begin{equation*}
  A := \beta_1  \begin{bmatrix}
      & \cos(\frac{1}{1000nk}) & \sin(\frac{1}{1000nk}) \\
      & - \sin(\frac{1}{1000nk}) & \cos(\frac{1}{1000nk}) \\
    \end{bmatrix},
\end{equation*}
where $\beta_1 > 0$ is a fixed constant. Then on the first layer, the Query-Key matrix is designed as
\begin{equation*} QK^{(1)} :=
     \begin{bmatrix}
        0_{(p+1)\times (p+1)} & 0 & 0 \\
        0 & A & 0 \\
        0 & 0 & 0\\
    \end{bmatrix} \in \rR^{D\times D},
\end{equation*}
which induces that with input matrix $M_0$, we have 
\begin{equation*}
    [M_0]_{(t_1,\cdot)}^T QK^{(1)} [M_0]_{(t_2,\cdot)} = \beta_1 \cdot \cos \left( \frac{t_1 - t_2 -1}{1000nk} \right),
\end{equation*}
for any $1 \le t_1, t_2 \le n(L+1) +k$. Then with softmax function on $M_0 QK^{(1)}M_0^T$, as well as the Value matrix 
\begin{equation*} 
 V^{(1)} :=
    \begin{bmatrix}
       0_{(p+3)\times (p+3)} & I_{(p+3)\times (p+3)} & 0_{(p+3)\times (D- 2p -6)}\\
       0 &  0 & 0\\
       0 & 0 & 0 \\
    \end{bmatrix}\in \rR^{D \times D},
\end{equation*}
sending $\beta_1 \to \infty$, we obtain the output on each row as
\begin{equation*}
\left[ \softmax \left( [M_0]_{(t, \cdot)} QK^{(1)} M_0^T \right) M_0 V^{(1)} \right]_{(t,\cdot)} =  [0, [M_0]_{(t-1, 1:p+3)}, 0]^T, \quad \forall 1 < t \le n(L+1) + k,
\end{equation*}
which refers that after the first Attention layer, the output matrix should be
\begin{align*}
      M_1 &=M_0 + \attn(M_0, QK^{(1)}, V^{(1)}) = [[M_0]_{(\cdot, 1:p+3)}, Z_1, [M_0]_{(\cdot,2(p+3)+1:D)} ] .
\end{align*}
 It implies that the first layer Attention head learn the first history feature $o_{i,L-1}$ for each observation $o_{i,L}$. Then on the second layer, we design the Query-Key matrix as
 \begin{equation*}
     QK^{(2)} := \begin{bmatrix}
         0_{(2p+4) \times (p+1)} & 0_{(2p+4) \times 2} &0_{(2p+4) \times (D - p - 3)} \\
         0_{2 \times (p+1)} & A & 0_{2 \times (D - p-3)} \\
         0_{(D - 2p -6) \times (p+1)} & 0_{(D - 2p -6) \times 2} & 0_{(D - 2p -6) \times (D - p - 3)}\\
     \end{bmatrix},
 \end{equation*}
as well as the Value matrix as
\begin{equation*}
  V^{(2)} :=
    \begin{bmatrix}
       0_{2(p+3)\times 2(p+3)} & I_{2(p+3)\times 2(p+3)} & 0_{2(p+3)\times (D - 4p - 12)}\\
       0 &  0 & 0\\
       0 & 0 & 0 \\
    \end{bmatrix}\in \rR^{D \times D},
\end{equation*}
which will induce the output on this layer as
 \begin{equation*}
     M_2 = M_1 + \attn (M_1, QK^{(2)}, V^{(2)}) = [[M_0]_{(\cdot, p+3)}, Z_1, Z_2, Z_3, [M_0]_{(\cdot, 4(p+3)+1:D)}].
 \end{equation*}
Repeating such construction $\cO(\ln L)$ times, we can obtain the $(L-1)$-step history (see Figure~\ref{fig:feature} for a detailed illustration). Now the output matrix should be
\begin{equation*}
    M_h = [[M_0]_{(\cdot, 1:p+3)}, Z_1, Z_2, Z_3, \dots, Z_{L-1}, [M_0]_{(\cdot,L(p+3)+1:D)}] \in \rR^{(n(L+1)+k) \times D}.
\end{equation*}
On the following layer, we consider the Query-Key matrix as
\begin{equation*}
    QK^{(f)} :=   \begin{bmatrix}
        0_{(p+1)\times (p+1)} & 0 & 0 \\
        0 & B & 0 \\
        0 & 0 & 0\\
    \end{bmatrix},
    \quad B := \beta_1  \begin{bmatrix}
      &  \cos(\frac{1}{1000nk}) & -\sin(\frac{1}{1000nk}) \\
      &  \sin(\frac{1}{1000nk}) & \cos(\frac{1}{1000nk}) \\
    \end{bmatrix},
\end{equation*}
and the value matrix is constructed as
\begin{equation*}
   V^{(f)} :=
    \begin{bmatrix}
       0_{(p+3)\times L(p+3)} & I_{(p+3)\times (p+3)} & 0_{(p+3)\times (D-(L+1)(p+3))}\\
       0 &  0 & 0\\
       0 & 0 & 0 \\
    \end{bmatrix}\in \rR^{D \times D},
\end{equation*}
which implies that sending $\beta_1 \to \infty$, the output on each row should be
\begin{equation*}
\left[ \softmax \left( [M_0]_{(t, \cdot)} QK^{(1)} M_0^T \right) M_0 V^{(1)} \right]_{(t,\cdot)} =  [0, [M_0]_{(t+1, 1:p+3)}, 0]^T, \quad \forall 1 \le t < n(L+1) + k,
\end{equation*}
So the output decouple matrix after this layer should be
\begin{equation*}
    M_{\text{dec}} = [[M_0]_{(\cdot, p+3)}, Z_1, Z_2, Z_3, \dots, Z_{L-1}, F_1, [M_0]_{(\cdot, (L+1)(p+3)+1:D)} ].
\end{equation*}
Then the decoupled feature learning process has been finished, which needs $\cO(\ln L)$ layers (see details in Figure~\ref{fig:feature}).

\paragraph{Gradient descent performing.} The following $\cO(T)$-layer $2p$-head architecture is designed to learn $\hat{\Pb}_L(\cdot | z_k)$ based on history information $\{Z_1, \dots, Z_{L-1}\}$. The construction follows immediately from Lemma~\ref{lem:reg}.
To be specific, from Eq.~\eqref{eq:prob_vec_n}, we need to take linear regression to estimate a matrix $\hat{W} \in \rR^{p \times p(L-1)}$.
Based on the $n$ samples collected, the estimation process is based on MSE loss, i.e, 
\begin{equation*}
    \arg \min_W \mathcal{L}(W) := \arg \min_W \sum_i \| o_{i,L} - W z_i \|_2^2,
\end{equation*}
where $z_i$ refers to the $(L-1)$-step history of $o_{i,L}$, which has been learned in previous layers. To perform such estimation process for $W$, we construct an $2p$-head $\cO(T)$-layer Attention. Each layer can perform one step gradient descent on $\mathcal{L}(W)$ with initial value $0$, and each row of $W$ is assigned to be learned by two heads independently (see Figure~\ref{fig:gd} for detailed illustration). 
Here we take the updating for $[W]_{(1,\cdot)}$ as an example, and denote the initial point as $0_{p(L-1)}$, which has been stored in $[M_{\text{dec}}]_{(t,(L+1)(p+3)+1:(L+1)(p+3)+p(L-1))}$ on each $1 \le t \le n(L+1)+k$. The gradient vector is
\begin{equation}\label{eq:gd_up}
\footnotesize
\begin{aligned}
   \partial \mathcal{L} / \partial W_{(1,\cdot)} &= 2 \sum_i (W_{(1,\cdot)}^T z_i - [o_{i,L}]_1) \cdot z_i\\
   &= 2 \sum_i \left( \relu (W_{(1,\cdot)}^T z_i - [o_{i,L}]_1) - \relu (- W_{(1,\cdot)}^T z_i + [o_{i,L}]_1) \right) \cdot z_i.  
\end{aligned}
\end{equation}
The construction will show that each attention layer is related to one-step gradient descent with learning rate $(L-1)^{-1}$, and the construction for each layer is the same. As the first two heads on each layer is related to the updating for $[W]_{(1,\cdot)}$, we design the first Attention head on each layer with Query-Key matrix as
\begin{equation*}
 \left([M_{\text{dec}}]_{(t_1, \cdot)} Q^{(g,1)} \right)^T  = \begin{bmatrix}
     [W]_{(1, \cdot)} \\
     -1 \\
    - \beta_2 \bm{1}_{p} \\
     0 \\
     - \beta_2\\
  \end{bmatrix}, \quad K^{(g,1)} [M_{\text{dec}}]_{(t_2, \cdot)}  = \begin{bmatrix}
      [Z_1]_{(t_2,1:p)} \\
      \cdots \\
      [Z_{L-1}]_{(t_2, 1:p)} \\
      [M_0]_{(t_2,1)} \\
      [F_1]_{(t_2,1:p)} \\
      0 \\
      \bm{1}(t_2 > n(L+1)) \\
  \end{bmatrix},
\end{equation*}
for any $1 \le t_1, t_2 \le n(L+1)+k$. Choosing $\beta_2 > 1000nk$, with ReLU activation function, we obtain
\begin{equation*}
\footnotesize
\relu \left(  [M_{\text{dec}}]_{(t_1, \cdot)}^T Q^{(g,1)}  K^{(g,1)} [M_{\text{dec}}]_{(t_2, \cdot)} \right)= \left\{
\begin{aligned}
    & \relu \left( [W]^\top_{(1, \cdot)} z'_{t_2} - [M_0]_{(t_2,1)} \right), \quad [F_1]_{(t_2,1:p+1)} = \od ,  \\
    & 0, \quad \text{otherwise},
\end{aligned}\right.
\end{equation*}
where we denote $z'_{t} := [[Z_1]_{(t_2,1:p)}^T, \dots, [Z_{L-1}]_{(t_2,1:p)}^T ]^T \in \rR^{p(L-1)} $.
Then with the Value matrix satisfying that
\begin{equation*}
    V^{(g,1)} [M_{\text{dec}}]_{(t_2, \cdot)}^T = \frac{1}{L-1}\begin{bmatrix}
        0 \\
        [Z_1]_{(t_2,1:p)} \\
      \cdots \\
      [Z_{L-1}]_{(t_2, 1:p)} \\
        0\\
    \end{bmatrix}, 
\end{equation*}
we can obtain the value on each row of the output matrix:
\begin{equation*}
  \left[  \attn \left( M_{\text{dec}}, Q^{(g,1)}, K^{(g,1)}, V^{(g,1)} \right)\right]_{(t,\cdot)} =\left[ 0, \frac{1}{L-1}\sum_i \relu \left( [W]^\top_{(1, \cdot)} z_{i} - [o_{i,L}]_1 \right) , 0 \right],
\end{equation*}
for any $1 \le t \le n(L+1)+k$. Also, we consider another Attention head for $W_{1, \cdot}$ with  $\{- Q^{(q,1)}, K^{(g,1)}, V_{(g,1)} \}$, the output on each row should be
\begin{equation*}
   \left[  \attn \left( M_{\text{dec}}, - Q^{(g,1)}, K^{(g,1)}, -V^{(g,1)} \right)\right]_{t,\cdot} =\left[ 0, - \frac{1}{L-1} \sum_i \relu \left( - [W]^\top_{(1, \cdot)} z_{i} + [o_{i,L}]_1 \right) , 0 \right].
\end{equation*}
Taking summation on both of the two heads, we can finish the update on $[W]_{(1, \cdot)}$ as in Eq.~\eqref{eq:gd_up}. The updates on other rows of $W$ are similar, so with such $2p$ Attention heads on each layer, we can finish one-step gradient descent on MSE loss by 
\begin{equation*}
    M_{\text{dec}} + \sum_{j=1}^p  \attn \left( M_{\text{dec}},  Q^{(g,j)}, K^{(g,j)}, V^{(g,j)} \right)+  \attn \left( M_{\text{dec}},  - Q^{(g,j)}, K^{(g,j)}, -V^{(g,j)} \right). 
\end{equation*}
Considering $\cO(T)$ layers with the same structure, we can obtain $\hat{W}$ with a small error. Now the output matrix should be
\begin{equation*}
    M_{\text{gd}} = [[M_0]_{(\cdot, p+3)}, Z_1, Z_2, Z_3, \dots, Z_{L-1}, F_1,[W]_{(1, \cdot)}, \dots, [W]_{(p, \cdot)}, [M_0]_{(\cdot, (L+1)(p+3)+ p^2(L-1)+1 : D)}].
\end{equation*}
\paragraph{Prediction with decoupled features.} Finally, on the last layer, we construct a $2p$-head Attention to make prediction on $\hat{\Pb}_L (\cdot | o_{\te,k-1}, \dots, o_{\te, k-L+1})$, and each dimension is corresponding to two Attention heads. To be specific, for the first dimension of $\hat{\Pb}_L (\cdot | o_{\te,k-1}, \dots, o_{\te, k-L})$, Attention head is designed with
\begin{align*}
&  \left([M_{\text{gd}}]_{(t_1, \cdot)} Q^{(pre,1)} \right)^T  = \begin{bmatrix}
     [Z_1]_{(t_2,1:p)} \\
      \cdots \\
      [Z_{L-1}]_{(t_2, 1:p)} \\
     0 \\
  \end{bmatrix}, \quad K^{(pre,1)} [M_{\text{gd}}]_{(t_2, \cdot)}^T  =  \begin{bmatrix}
      [W]_{(1, \cdot)} \\
      0 \\
  \end{bmatrix}, \\
  & V^{(pre,1)} [M_{\text{gd}}]_{(t_2, \cdot)}^T = \begin{bmatrix}
      \frac{1}{n(L+1)+k} \\
      0\\
  \end{bmatrix}.
\end{align*}
Then we will obtain 
\begin{equation*}
   \left[  \attn \left( M_{\text{gd}}, Q^{(pre,1)}, K^{(pre,1)}, V^{(pre,1)} \right)\right]_{(n(L+1)+k,\cdot)} =\left[ \relu \left( [W]^\top_{(1, \cdot)} o_{\te,k-L+1:k-1} \right) , 0 \right],
\end{equation*}
and 
\begin{align*}
 & \quad  \left[ \attn \left( M_{\text{gd}},  Q^{(pre,1)}, K^{(pre,1)}, V^{(pre,1)} \right)+  \attn \left( M_{\text{gd}},  -Q^{(pre,1)}, K^{(pre,1)}, -V^{(pre,1)}\right) \right]_{(n(L+1)+k,\cdot)}\\
 &= \left[   [W]^\top_{(1, \cdot)} o_{\te, k-L+1 : k-1}  , 0 \right],   
\end{align*}
which finish the prediction on $\hat{\Pb}_L (o_{\te,k} = e_1 | o_{\te,k-1}, \dots, o_{\te, k-L})$. The constructions on other $2p-2$ heads are similar.

\paragraph{Optimization error.} Then we turn to the approximation for $\epsilon_3$, which is induced by the finite gradient steps ($\cO(T)$ steps) the transformer performs. The error could be estimated directly from Lemma~\ref{lem:reg}. Denoting
\begin{equation*}
  Z=[o_{1,1:L-1},\ldots,o_{n,1:L-1}] \in\rR^{p(L-1)\times n},
\end{equation*}
from Assumption~\ref{as:eigen_value}, we have
\begin{equation*}
    \alpha \le \lambda_{\min} \left( \frac{1}{n} ZZ^T \right) \le \lambda_{\max} \left( \frac{1}{n} ZZ^T \right) \le L, \quad \| o_{\te,k-L+1:k-1} \|_2 = \sqrt{L-1}, \quad \| [W_*]_{(j, \cdot)} \|_2 = \cO(1),
\end{equation*}
so 
\begin{equation*}
    \epsilon_3 = \cO \left(e^{- \alpha T /(2 L) } p L^{1/2} \max_{j \in [p]} \| [W_*]_{(j,\cdot)} \|_2 \right) = \cO(p L^{1/2} e^{- \alpha T /(2 L)}).
\end{equation*}
\subsection{Generalization Error}

For $\epsilon_2$, we can express the solution $\hat{W}$ for Eq.~\eqref{eq:prob_vec_n} as
\begin{equation*}
    \hat{W} := O Z^T (ZZ^T )^{-1},
\end{equation*}
where we use the notation
\begin{equation*}
    O := \begin{bmatrix}
        o_{1,L} & o_{2,L} & \cdots & o_{n,L} 
    \end{bmatrix} \in \rR^{p \times n},\quad Z=[o_{1,1:L-1},\ldots,o_{n,1:L-1}] \in\rR^{p(L-1)\times n}. 
\end{equation*}
Denoting $z_\te := o_{\te,k-L+1:k-1}$ and $\Delta := O - W_* Z$, we have
\begin{align}\label{eq:b}
 \epsilon_2 &= \E_{o_{\te, 1:k-1}} \| \Pb_L(\cdot| o_{\te,k-L+1:k-1}) - \hat{\Pb}_L(\cdot| o_{\te,k-L+1:k-1})\|_1 \notag\\
 &= \sum_{j=1}^p \E_{z_{\te}} \left| \big(  [W_*]_{(j,\cdot)}^T - [O]_{(j,\cdot)}^T Z^T(ZZ^T)^{-1}  \big) z_{\te}   \right| \notag\\
 &\le \sum_{j=1}^p \sqrt{L} \| [W_*]_{(j,\cdot)} - [O]_{(j,\cdot)} Z^T(ZZ^T)^{-1} \|_2 \notag\\
 &= \sum_{j=1}^p \sqrt{L} \| [W_*]_{(j,\cdot)} - \left( [W_*]_{(j,\cdot)} Z + [\Delta]_{(j,\cdot)} \right) Z^T(ZZ^T)^{-1} \|_2 \notag\\
 &= \sqrt{L} \sum_{j=1}^p \| [\Delta]_{(j,\cdot)} Z^T (ZZ^T)^{-1} \|_2\notag\\
 &\le \frac{\sqrt{L}}{n\alpha} \sum_{j=1}^p \| \big([\Delta]_{(j,\cdot)} - \E_i[[\Delta]_{(j,\cdot)}] + \E_i[[\Delta]_{(j,\cdot)}]\big) Z^T \|_2\notag\\
 &\le \frac{\sqrt{L}}{n\alpha} \sum_{j=1}^p \| \big([\Delta]_{(j,\cdot)} - \E[[\Delta]_{(j,\cdot)}] \big) Z^T \|_2 + \frac{\sqrt{L}}{n\alpha} \sum_{j=1}^p \| \E[[\Delta]_{(j,\cdot)}]\big) Z^T \|_2
\end{align}
where the first inequality uses the Cauchy-Schwartz inequality, and the second inequality is from Assumption \ref{as:eigen_value}, where the expectation $\E[[\Delta]_{(j,i)}]=\E_{o_{i,1:k-1}}[\Pb(e_j\mid o_{1:k-1}) - \Pb_L(e_j\mid o_{k-L+1:k-1})]$ due to the decomposition:
\begin{align*}
    [\Delta]_{(j,i)} =& [O]_{(j,i)} - [W_*]_{(j,\cdot)} o_{i,1:L-1}\\
    =& \bm{1}(o_{i,L} = e_j) - \Pb(e_j|o_{i,1:L-1}) + \Pb(e_j|o_{i,1:L-1}) - \Pb_L(e_j|o_{i,1:L-1}).
\end{align*}
Hence, we can deal with the second term above:
\begin{align*}
    \frac{\sqrt{L}}{n\alpha} \sum_{j=1}^p \| \E[[\Delta]_{(j,\cdot)}]\big) Z^T \|_2 \le& \frac{L}{n\alpha} \sum_{j=1}^p \sum_{i=1}^n \E_{o_{i,1:k-1}} |\Pb(e_j\mid o_{1:k-1}) - \Pb_L(e_j\mid o_{k-L+1:k-1})|\\
    =& \frac{L}{n\alpha}\sum_{i=1}^n \E_{o_{i,1:k-1}} \big\|\Pb(\cdot\mid o_{i,1:k-1}) - \Pb_L(\cdot\mid o_{i,k-L+1:k-1})\big\|_1\\
    \le& \cO(\frac{Ld}{\alpha} \cdot e^{-L\gamma^4}),
\end{align*}
where the first inequality uses the formulation that $\|[Z]_{(i,\dot)}\|_2\le\sqrt{L}$, and the second inequality uses Lemma \ref{lem:model approximation}.

Next, for the first term in \eqref{eq:b}, 
we can define the error $\delta_{j,i}:=[\Delta]_{(j,i)} - \E[[\Delta]_{(j,i)}]$. For each $i,j$, $\delta_{j,i}$ is a zero-mean $1$-sub-Gaussian variable. We also have for each $i$, $\max\{\|z_i z_i^\top\|_2, \|z_i^\top z_i\|_2\}\le L$. Thus, we can invoke Lemma \ref{lm:Matrix_Sub-Gaussian_Series} to obtain that with probability at least $1 - \frac{1}{n}$, for any $j=1,\ldots,p$,
\begin{align*}
\| \big([\Delta]_{(j,\cdot)} - \E[[\Delta]_{(j,\cdot)}] \big) Z^T \|_2 = \| \sum_{i=1}^n \delta_{i,j}z_i \|_2 \le 4\sqrt{nL \ln(2 nLp^2)}.
\end{align*}
Therefore, by taking the results above back into \eqref{eq:b}, we can obtain that
\begin{align*}
    \epsilon_2 \le \cO\Big(\frac{p L\sqrt{ \ln(nLp)}}{\sqrt{n} \alpha} + \frac{Ld}{\alpha} \cdot e^{-L\gamma^4}\Big).
\end{align*}

\section{Proof sketches for Theorem~\ref{thm:extend}}\label{pf:thm_extend}
We also decompose the prediction error into three parts as in \eqref{eq:decompose} and analyze them correspondingly.

\subsection{Model Approximation}
For the model approximation error $\epsilon_1$, under Assumption \ref{as:under-complete_gam_obs}, we can also approximate the $m$-step transition probability $\Pb(o_{k:k+m}\mid o_{1:k-1})$ by a $(L-1)$-memory probability $\hat\Pb_L(o_{k:k+m}\mid o_{k-L+1:k-1})$. Since we can take $o_{k:k+m}$ as a whole vector, with similar techniques in Section \ref{s:Low-rank HMM}, we can show that
\begin{lemma}\label{lem:eps1_extend}
    For any $\epsilon>0$, there exists a $\cO(L)$-memory transition probability $\hat{\Pb}_L$ with $L=\Theta(\gamma^{-4}\log(d/\epsilon)$ such that
    \[
    \E_{o_{1:k}} \big\|\Pb(o_{k:k+m}\mid o_{1:k}) - \Pb_L(o_{k:k+m}\mid o_{t-L:t})\big\|_1 \le \cO \left( d e^{-L\gamma^4} \right).
    \]
\end{lemma}
This model approximation bound is the same to Lemma \ref{lem:model approximation}, and the $\Pb_L$ also enjoys the low-rank structure
\begin{align*}
    \Pb_L(o_{k:k+m}\mid o_{k-L+m:k-1}) :=& \mu(o_{k:k+m})^\top \phi(o_{k-L+m:k-1}),
\end{align*}
where $\mu(o_{k:k+m}),\phi(o_{k-L+m:k-1})\in\rR^d$ are representation vectors. For conciseness, we defer the details to Appendix \ref{ss:Model Approximation Error_extend}.

After embedding the $m$-step observation $o_{k:k+m}$ as one-hot vector $\vecc (o_{k:k+m}) \in \rR^{p^m}$, we can express the mapping function $\mu(\cdot)$ as
\begin{equation*}
  \mu(o_{k:k+m}) = U' \vecc (o_{k:k+m}),
\end{equation*}
where $U' \in \rR^{d \times p^m}$. Considering the linear assumption on $\phi$, similar to Eq.~\eqref{eq:prob_vec}, we can also obtain
\begin{equation*}
    \Pb_L(\cdot | o_{k-L+m:k-1}) := W'_* o_{k-L+m:k-1},
\end{equation*}
for some $W'_* \in \rR^{p^m \times p(L-m)}$. Taking decomposition for the approximation error, we have
\begin{align*}
& \quad \E_{o_{\te, 1:k-1}} \| \Pb (o_{\te,k:k+m-1} | o_{\te,1:k-1}) - \mathrm{read}(\mathrm{TF}_\theta(M_0)) \|_1 \\
&\le \underbrace{ \E_{o_{\te, 1:k-1}} \| \Pb (o_{\te,k:k+m-1} | o_{\te,1:k-1}) - \Pb_L(o_{\te,k:k+m-1}| o_{\te,k-L+m:k-1}) \|_1}_{\epsilon_1: \mathrm{model~approximation}} \\
& \quad + \underbrace{\E_{o_{\te, 1:k-1}} \|  \Pb_L(o_{\te,k:k+m-1}| o_{\te,k-L+m:k-1}) - \hat{\Pb}_L(o_{\te,k:k+m-1}| o_{\te,k-L+m:k-1})\|_1}_{\epsilon_2 : \mathrm{generalization}} \\
& \quad + \underbrace{ \E_{o_{\te, 1:k-1}} \|\hat{\Pb}_L(o_{\te,k:k+m-1}| o_{\te,k-L+m:k-1})  - \mathrm{read}(\mathrm{TF}_\theta(M_0)) \|_1 }_{\epsilon_3: \mathrm{optimization}},
\end{align*}
where $\hat{\Pb}_L(\cdot| o_{\te,k-L+1:k-1})$ refers to the solution based on $n$ samples we collected:
\begin{align*}
   &  \hat{\Pb}_L(\cdot| o_{\te,k-L+m:k-1}) = \hat{W}' [o_{\te,k-L+m}, \dots, o_{\te, k-1}]^T,\\
   & \hat{W}' := \arg \min_W \sum_i \| \vecc (o_{i,L-m+1:L}) - W o_{i,1:L-m} \|_2^2 .
\end{align*}
In the error decomposition, $\epsilon_1 = \cO(d e^{- \gamma^4 L})$ can be obtained from Lemma~\ref{lem:eps1_extend} immediately. And in further analysis, we will estimate $\epsilon_2$ and $\epsilon_3$ respectively.

\subsection{Transformer Construction}
Then the construction is similar to the construction for Theorem~\ref{thm:main}. So here we just provide a sketch for it.

\paragraph{Decoupled feature learning.} Recalling the matrix:
\begin{equation*}
  A := \beta_1  \begin{bmatrix}
      &  \cos(\frac{1}{1000nk}) & \sin(\frac{1}{1000nk}) \\
      & - \sin(\frac{1}{1000nk}) & \cos(\frac{1}{1000nk}) \\
    \end{bmatrix}, \quad  B := \beta_1  \begin{bmatrix}
      &  \cos(\frac{1}{1000nk}) & -\sin(\frac{1}{1000nk}) \\
      &  \sin(\frac{1}{1000nk}) & \cos(\frac{1}{1000nk}) \\
    \end{bmatrix},
\end{equation*}
on each time index $t$, we can use $A$ to capture the history information $Z_r$, and use $B$ to capture the future information $F_r$. So with $\cO(\ln (L-m) + \ln m) = \cO(\ln L)$ layers, we can obtain the output matrix as
\begin{equation*}
    M_{\text{dec}} = [ [M_0]_{(\cdot, 1:p+3)}, Z_1, Z_2, \dots, Z_{L-m}, F_1, F_2, \dots, F_m, [M_0]_{(\cdot,(L+1)(p+3)+1 :D)}].
\end{equation*}
Then before taking gradient descent, we use the one-hot mapping function $\vecc$ on each row of $\{ [M_0]_{(\cdot, 1:p)}, [F_1]_{(\cdot,1:p)}, \dots, [F_{m-1}]_{(\cdot,1:p)}\}$, which refers to the current and future observations on each time index. After that, we will obtain
\begin{equation*}
    M_v := [ [M_0]_{\cdot, 1:p+3}, Z_1, Z_2, \dots, Z_{L-m}, F_1, F_2, \dots, F_m, H ,[M_0]_{(\cdot, (L+1)(p+3)+ p^m + 1 :D )}],
\end{equation*}
where
\begin{equation*}
    [H]_{(t,\cdot)} = \vecc \left[ [M_0]_{(t, 1:p)}, [F_1]_{(t,1:p)}, \dots, [F_{m-1}]_{(t,1:p)}   \right]^T 
\end{equation*}
for each $1 \le t \le nL+n+k$.

\paragraph{Gradient descent and final prediction.} After obtaining these features, we shall perform gradient descent on MSE loss
\begin{equation*}
    \arg \min _{W'} \sum_i \| \vecc (o_{i,L-m+1:L}) - W' o_{i,1:L-m} \|_2^2.
\end{equation*}
Then we could use $2 p^m$-head $\cO(T)$-layer Attention to perform the gradient descent on $W$, in which the feature $H$ and $\{ Z_1, \dots, Z_{L-m} \}$ will be taken into consideration. The construction is similar to Theorem~\ref{thm:main}.

\paragraph{Optimization error.} For $\epsilon_3$, under Assumption~\ref{as:eigen_value}, we can also use Lemma~\ref{lem:opt} to obtain that
\begin{equation*}
    \epsilon_3 = \cO \left( p^m L^{1/2} e^{- \alpha T / (2L)}  \right).
\end{equation*}

\subsection{Generalization Error}
We can rewrite $\hat{\Pb}_L(\cdot| o_{\te,k-L+1:k-m})$ as
\begin{equation*}
    \hat{\Pb}_L (\cdot |o_{\te,k-L+m:k-1}) = \hat{W}'  o_{\te,k-L+m:k-1}, \quad \hat{W}' = O_m Z_m^T (Z_m Z_m^T)^{-1},
\end{equation*}
where we denote
\begin{align*}
  &  O_m := \begin{bmatrix}
       \vecc ( o_{1,L-m+1:L} ) & \vecc ( o_{2,L-m+1:L} ) & \cdots & \vecc ( o_{n,L-m+1:L} ) 
    \end{bmatrix} \in \rR^{p^m \times n}, \\
  &  Z_m := \begin{bmatrix}
       o_{1,1:L-m} & o_{2,1:L-m} & \cdots & o_{n,1:L-m} \\
    \end{bmatrix} \in \rR^{(L-m) \times n}. 
\end{align*}
Denoting $z_\te := o_{\te,k-L+m:k-1}$ and $\Delta := O_m - W'_* Z_m$, we have
\begin{align}\label{eq:b_extend}
 \epsilon_2 &= \E_{o_{\te, 1:k-1}} \| \Pb_L(\cdot| o_{\te,k-L+m:k-1}) - \hat{\Pb}_L(\cdot| o_{\te,k-L+m:k-1})\|_1 \notag\\
 &= \sum_{j=1}^{p^m} \E_{z_{\te}} \left| \big(  [W'_*]_{(j,\cdot)}^T - [O_m]_{(j,\cdot)}^T Z_m^T(Z_m Z_m^T)^{-1}  \big) z_{\te}   \right| \notag\\
 &\le \sum_{j=1}^{p^m} \sqrt{L} \| [W'_*]_{(j,\cdot)} - [O_m]_{(j,\cdot)} Z_m^T(Z_m Z_m^T)^{-1} \|_2 \notag\\
 &= \sum_{j=1}^{p^m} \sqrt{L} \| [W'_*]_{(j,\cdot)} - \left( [W'_*]_{(j,\cdot)} Z_m + [\Delta]_{(j,\cdot)} \right) Z_m^T(Z_m Z_m^T)^{-1} \|_2 \notag\\
 &= \sqrt{L} \sum_{j=1}^{p^m} \| [\Delta]_{(j,\cdot)} Z_m^T (Z_m Z_m^T)^{-1} \|_2\notag\\
 &\le \frac{\sqrt{L}}{n\alpha} \sum_{j=1}^{p^m} \| \big([\Delta]_{(j,\cdot)} - \E_i[[\Delta]_{(j,\cdot)}] + \E_i[[\Delta]_{(j,\cdot)}]\big) Z_m^T \|_2\notag\\
 &\le \frac{\sqrt{L}}{n\alpha} \sum_{j=1}^{p^m} \| \big([\Delta]_{(j,\cdot)} - \E[[\Delta]_{(j,\cdot)}] \big) Z_m^T \|_2 + \frac{\sqrt{L}}{n\alpha} \sum_{j=1}^{p^m} \| \E[[\Delta]_{(j,\cdot)}]\big) Z_m^T \|_2,
\end{align}
where the first inequality uses the Cauchy-Schwartz inequality, and the second inequality is from Assumption \ref{as:eigen_value}, where the expectation $\E[[\Delta]_{(j,i)}]=\E_{o_{i,1:k-1}}[\Pb(e_j\mid o_{1:k-1}) - \Pb_L(e_j\mid o_{k-L+m:k-1})]$ due to the decomposition:
\begin{align*}
    [\Delta]_{(j,i)} =& [O]_{(j,i)} - [W'_*]_{(j,\cdot)} o_{i,1:L-m}\\
    =& \bm{1}(o_{i,L-m+1 : L} = e_j) - \Pb(e_j|o_{i,1:L-m}) + \Pb(e_j|o_{i,1:L-m}) - \Pb_L(e_j|o_{i,1:L-m}).
\end{align*}
Hence, we can deal with the second term above:
\begin{align*}
    \frac{\sqrt{L}}{n\alpha} \sum_{j=1}^{p^m} \| \E[[\Delta]_{(j,\cdot)}]\big) Z_m^T \|_2 \le& \frac{L}{n\alpha} \sum_{j=1}^{p^m} \sum_{i=1}^n \E_{o_{i,1:k-1}} |\Pb(e_j\mid o_{1:k-1}) - \Pb_L(e_j\mid o_{k-L+m:k-1})|\\
    =& \frac{L}{n\alpha}\sum_{i=1}^n \E_{o_{i,1:k-1}} \big\|\Pb(\cdot\mid o_{i,1:k-1}) - \Pb_L(\cdot\mid o_{i,k-L+m:k-1})\big\|_1\\
    \le& \cO(\frac{Ld}{\alpha} \cdot e^{-L\gamma^4}),
\end{align*}
where the first inequality uses the formulation that $\|[Z_m]_{(i,\dot)}\|_2\le\sqrt{L}$, and the second inequality uses Lemma \ref{lem:eps1_extend}.

Next, for the first term in \eqref{eq:b_extend}, 
we can define the error $\delta_{j,i}:=[\Delta]_{(j,i)} - \E[[\Delta]_{(j,i)}]$. For each $i,j$, $\delta_{j,i}$ is a zero-mean $1$-sub-Gaussian variable. We also have for each $i$, $\max\{\| [Z_m]_{(\cdot,i)} [Z_m]_{(\cdot,i)}^\top\|_2, \|[Z_m]_{(\cdot,i)}^\top[Z_m]_{(\cdot,i)}\|_2\}\le L$. Thus, we can invoke Lemma \ref{lm:Matrix_Sub-Gaussian_Series} to obtain that with probability at least $1-\frac{1}{n}$, for any $j=1,\ldots,p^m$,
\begin{align*}
\| \big([\Delta]_{(j,\cdot)} - \E[[\Delta]_{(j,\cdot)}] \big) Z^T \|_2 = \| \sum_{i=1}^n \delta_{i,j}[Z_m]_{(\cdot,i)} \|_2 \le 4\sqrt{nL \ln(2 nLp^{m+1})}.
\end{align*}
Therefore, by taking the results above back into \eqref{eq:b_extend}, we can obtain that
\begin{align*}
    \epsilon_2 \le \cO\Big(\frac{p^m L\sqrt{ \ln(nLp)}}{\sqrt{n} \alpha} + \frac{Ld}{\alpha} \cdot e^{-L\gamma^4}\Big).
\end{align*}

\section{Proof for Lemma~\ref{lem:model approximation}}\label{ss: Proof for Model Approximation Error}

To facilitate analysis, we define the belief state $b_k(o_{1:k-1})\in\Delta(\cH)$ as the posterior given observations: $b_k(o_{1:k})(h) = \Pb(h_k\mid o_{1:k}).$
Combining this notation and the low-rank hidden-state transition, we can write 
\begin{align*}
    \Pb(o_k\mid o_{1:k-1}) =& \sum_{h_k,h_{k-1}}\Pb(o_k\mid h_k)\Pb(h_k\mid h_{k-1})\Pb(h_{k-1}\mid o_{1:k-1})\\
    =& \Big(\sum_{h_k} \bbT(o_k\mid h_k) w^*(h_k)\Big)^\top\cdot \Big(\sum_{h_{k-1}} \psi^*(h_{k-1})b(o_{1:k-1})(h_{k-1})\Big).
\end{align*}
The transition is the inner product of $d$-dimensional representations of history $o_{1:k-1}$ and next token $o_k$. Especially, the historical information is embedded into the belief state. Thus, to approximate $\Pb$ by $\Pb_L$, we need to approximate $b(o_{1:k-1})$ by a $(L-1)$-memory belief state $b_L(o_{k-L+1:k-1})$. Assumption \ref{as:under-complete_gam_obs} implies that we can reverse the inequality to obtain the contraction from observation to hidden state distributions
\begin{equation*}
    \|d-d'\|_1 \le \gamma^{-1} \|\bbT d - \bbT d'\|_1.
\end{equation*}
Hence, by constructing a history-independent belief state $\tilde b_0$ within a KL-ball of $b$: $\mathrm{KL}(b,\tilde b_0) \le d^3$ (which can be realized by G-optimal design), the belief state $b_L(o_{k-L+1:k-1})$ induced from $\tilde b_0$ can gradually approximate $b(o_{1:k-1})$ that has the same $(L-1)$-length observations. Theorem 14 of \citet{uehara2022provably} demonstrated that
\begin{lemma}[Theorem 14 of \citet{uehara2022provably}]\label{lm:Exponential Stability for Low-rank Transition}
    Under Assumption \ref{as:under-complete_gam_obs}, for $K\ge L+1$, $L\ge C\gamma^{-4}\log(d/\epsilon)$, where $C>0$ is a constant, we have
    \begin{equation}\label{eq:Exponential Stability}
        \E_{o_{1:k-1}} \big\| b(o_{1:k-1}) - b_L(o_{k-L+1:k-1}) \big\|_1 \le \epsilon.
    \end{equation}
\end{lemma}

\begin{proof}[Proof of Lemma \ref{lem:model approximation}]
The proof is the same to Proposition 7 of \citet{guo2023provably}. The only difference is there is no actions in HMM. For any $k\ge L+1$, given the $b_L$ satisfying \eqref{eq:Exponential Stability}, now, we can construct the probability as
\begin{align*}
    \Pb_L(o_k\mid o_{k-L+1:k-1}) =& \Big(\sum_{h_k} \bbT(o_k\mid h_k) w^*(h_k)\Big)^\top\cdot \Big(\sum_{h_{k-1}} \psi^*(h_{k-1})b_L(o_{k-L+1:k-1})(h_{k-1})\Big)\\
    :=& \mu(o_k)^\top \phi(o_{k-L+1:k-1}),
\end{align*}
where we use the notation 
\begin{align*}
    \mu(o_k) = \sum_{h_k} \bbT(o_k\mid h_k) w^*(h_k),\quad \phi(o_{k-L+1:k-1}) = \sum_{h_{k-1}} \psi^*(h_{k-1})b_L(o_{k-L+1:k-1})(h_{k-1}).
\end{align*}

Hence, we deduce that
\begin{align}
    \E_{o_{1:k-1}} \Pb(o_k\mid o_{1:k-1}) =& \E_{o_{1:k-1}} \sum_{h_k} \bbT(o_k\mid h_k) w^*(h_k)^\top \cdot \sum_{h_{k-1}} \psi^*(h_{k-1})b(o_{1:k-1})(h_{k-1})\notag\\
    \le& \E_{o_{1:k-1}} \sum_{h_k} \bbT(o_k\mid h_k) w^*(h_k)^\top\notag\\
    &\quad \cdot\sum_{h_{k-1}} \psi^*(h_{k-1}) \Big(\big|b(o_{1:k-1})(h_{k-1}) - b_L(o_{k-L+1:k-1})(h_{k-1})\big| + b_L(o_{k-L+1:k-1})(h_{k-1})\Big)\notag\\
    =& \E_{o_{1:k-1}} \sum_{h_k} \bbT(o_k\mid h_k) w^*(h_k)^\top \cdot \sum_{h_{k-1}} \psi^*(h_{k-1})b_L(o_{k-L+1:k-1})(h_{k-1})\notag\\
    &\quad + \E_{o_{1:k-1}} \sum_{h_k} \bbT(o_k\mid h_k) w^*(h_k)^\top \cdot\sum_{h_{k-1}} \psi^*(h_{k-1})\big|b(o_{1:k-1})(h_{k-1}) - b_L(o_{k-L+1:k-1})(h_{k-1})\big|.\label{eq:a}
\end{align}
Since we have for any $h_{k-1}$
\begin{align*}
    \sum_{h_k} \bbT(o_k\mid h_k) w^*(h_k)^\top \psi^*(h_{k-1}) = \sum_{h_k} \Pb(o_k\mid h_k) \Pb(h_k\mid h_{k-1}) \le 1,
\end{align*}
term \eqref{eq:a} can be bounded as
\begin{align*}
    &\E_{o_{1:k-1}} \sum_{h_{k-1}} \Big(\sum_{h_k} \bbT(o_k\mid h_k) w^*(h_k)^\top \cdot \psi^*(h_{k-1})\Big)\cdot \big|b(o_{1:k-1})(h_{k-1}) - b_L(o_{k-L+1:k-1})(h_{k-1})\big|\\
    \le& \E_{o_{1:k-1}} \big|b(o_{1:k-1})(h_{k-1}) - b_L(o_{k-L+1:k-1})(h_{k-1})\big|\\
    \le& \epsilon,
\end{align*}
where the first inequality is by the Cauchy-Schwarz inequality, and the second inequality uses Lemma \ref{lm:Exponential Stability for Low-rank Transition}. Therefore, we obtain
\begin{align*}
     \E_{o_{1:k-1}} \Pb(o_k\mid o_{1:k-1}) \le  \E_{o_{1:k-1}} \Pb_L(o_k\mid o_{k-L+1:k-1}) + \epsilon,
\end{align*}
which concludes the proof.
\end{proof}

Then, we can construct the $(L-1)$-memory probability by replacing the belief state 
\begin{align*}
    \Pb_L(o_k\mid o_{k-L+1:k-1}) =& \Big(\sum_{h_k} \bbT(o_k\mid h_k) w^*(h_k)\Big)^\top\cdot \Big(\sum_{h_{k-1}} \psi^*(h_{k-1})b_L(o_{k-L+1:k-1})(h_{k-1})\Big)\\
    :=& \mu(o_k)^\top \phi(o_{k-L+1:k-1}),
\end{align*}

\section{Proof for Lemma \ref{lem:eps1_extend}}\label{ss:Model Approximation Error_extend}

\begin{proof}[Proof of Lemma \ref{lem:eps1_extend}]
Under the operator $\bbM$, we can write
\[
\Pb(o_{k:k+m}\mid h_t) = \int_{\cH} \bbM(o_{k:k+m}\mid h_{t+1}) w^*(h_{t+1})^\top \psi^*(h_t) \md h_t.
\]
We wish to approximate 
\[
\Pb(o_{k:k+m}\mid o_{1:h})~\text{by}~ \Pb_L(o_{k:k+m}\mid o_{t-L+1:t}).
\]
Given a history observation $o_{1:k}$, we define the belief state $b_t(o_{1:k})\in \Delta(\cS)$ as the distribution
\[
b_t(o_{1:k})(h) = \Pb(h_k=h\mid o_{1:k}).
\]
Additionally, for any distribution $b\in\Delta(\cS)$, we define the belief update operator $B_{k-1}(b,o_{k:k+m})$ as
\[
 B_{k-1}(b,o_{k:k+m})(h) = \frac{\bbM(o_{k:k+m}\mid h)\sum_{h'} b(h')\Pb(h|h')}{\sum_{h''}\bbM(o_{k:k+m}\mid h'')\sum_{h'} b(h')\Pb(h''|h')}.
\]
then, the update for belief state is
\[
b(o_{1:k-1}) = B(b_{k-1}(o_{1:k-1}),o_{k:k+m}).
\]

Given this notation, we can write $\Pb$ as
\begin{equation}
    \Pb(o_{k:k+m}\mid o_{1:k-1}) = \Big(\int_{\cH} \bbM(o_{k:k+m}\mid h_{t+1}) w^*(h_{t+1}) \md h_{t+1}\Big)^\top \cdot \int_{\cH} \psi^*(h_{k-1}) b(o_{1:k-1})(h_{k-1}) \md h_{k-1}.
\end{equation}
Thus, to approximate $\Pb$ by $\Pb_L$, it suffices to approximate $b(o_{1:k})$ by some belief state $b_L(o_{t-L:t})$.

To construct a good approximation, we can first construct a history-independent belief distribution $\tilde b_0\in\Delta(\cS)$ by G-optimal design \citep{uehara2022provably} such that for any belief state
\begin{equation}
    \mathrm{KL}(b,\tilde b_0) \le d^3.
\end{equation}

\begin{lemma}[Exponential Stability for Low-rank Transition]\label{lm:Exponential Stability for Low-rank Transition}
    Under Assumption \ref{as:under-complete_gam_obs}, for $L\ge C\gamma^{-4}\log(d/\epsilon)$, we have
    \begin{align*}
        \E \big\| b(o_{1:k-1}) - b_L(o_{t:t+L}) \big\|_1 \le \epsilon.
    \end{align*}
\end{lemma}

Then, by following the same analysis as the proof of Lemma \ref{lem:model approximation}, we can prove the desired result.
\end{proof}

\section{Technical Lemmas}

\begin{lemma}[Convergence rate in gradient descent]\label{lem:opt}
 Suppose $L $ is $\alpha$-strongly convex and $\beta$-smooth for some $0 < \alpha < \beta$. Then the gradient descent iterates $w_{GD}^{t+1} := w_{GD}^t - \eta \nabla L(w_{GD}^t)$ with learning rate $\eta = \beta^{-1}$ and initialization $w_{GD}^0$ satisfies 
 \begin{align*}
    & \| w_{GD}^t - w^* \|_2^2 \le e^{- t / \kappa} \cdot \| w_{GD}^0 - w^* \|_2^2, \\
    & L(w_{GD}^t) - L(w^*) \le \frac{\beta}{2} e^{- t / \kappa} \cdot \| w_{GD}^0 - w^* \|_2^2 , 
 \end{align*}
 where $\kappa = \beta / \alpha$ is the condition number, and $w^* = \arg \min L(w)$ is the optimizer of function $L-1$.
\end{lemma}

\begin{lemma}[Lemma G.2 in \citet{ye2023corruption}, Theorem 2.29 in \citet{zhang_2023_ltbook}]\label{lem:concen}
Let $\{ \epsilon_t \}$ be a sequence of zero-mean conditional $\sigma$-subGaussian random variable, i.e, $\ln \E[e^{\lambda \epsilon_i} | \mathcal{S}_{i-1}] \le \lambda^2 \sigma^2 / 2 $, where $\mathcal{S}_{i-1}$ represents the history data. With probability at least $1 - \delta$, for any $t \ge 1$, we have
\begin{equation*}
    \sum_{i=1}^t \epsilon_i^2 \le 2t\sigma^2 + 3 \sigma^2 \ln(1 / \delta).
\end{equation*}
    
\end{lemma}

\begin{lemma}[Theorem 4 in \citet{bai2023transformers}]\label{lem:reg}
For any $\lambda \ge 0$, $0 \leq \alpha \leq \beta$ with 
\[
\kappa := \frac{\beta+\lambda}{\alpha+\lambda},
\]
$B_w > 0$, and $\varepsilon < \frac{B_x B_w}{2}$, there exists an $L$-layer attention-only transformer $TF^0_\theta$ with
\[
M = \lceil 2\kappa \log(B_x B_w / (2\varepsilon)) \rceil + 1
\]
(With $R := \max \{B_x B_w, B_y, 1\}$) such that the following holds. On any input data $(D, x_{N+1})$ such that the regression problem is well-conditioned and has a bounded solution:
\[
\alpha \leq \lambda_{\min}(X^\top X / N) \leq \lambda_{\max}(X^\top X / N) \leq \beta,
\]
\[
\|w^{\lambda}_{\mathrm{ridge}}\|_2 \leq B_w/2,
\]
$TF^0_\theta$ approximates the prediction $\hat{y}_{N+1}$ as
\[
\left| \hat{y}_{N+1} - \langle w^{\lambda}_{\mathrm{ridge}}, x_{N+1} \rangle \right| \leq \varepsilon.
\]  
\end{lemma}

\begin{lemma}[Lemma F.3 of \citet{fan2023provably}]\label{lm:Matrix_Sub-Gaussian_Series}
Consider a sequence of matrix $\{A_t\}_{t=1}^{\infty}$ with dimension $d_1\times d_2$ and an i.i.d. sequence $\{\epsilon_t\}_{t=1}^{\infty}$, where $\epsilon_t$ is conditional $\sigma$-subgaussian (i.e., $\E(e^{\alpha\epsilon_t} \,|\, A_t) \le e^{\alpha^2\sigma^2/2}$ almost surely for all $\alpha\in\rR$). Define the matrix sub-Gaussian series
$S = \sum_{t=1}^n \epsilon_t A_t$ with bounded matrix variance statistic:
$$
\max\left\{\left\|A_tA_t^{\top}\right\|_{op},\left\|A_t^{\top}A_t\right\|_{op}\right\}\le v_t.
$$
Then, for all $u>0$, we have 
$$
\Pb\left(\|S\|_{op} \ge u\right) \le (d_1+d_2)\exp\Big(-\frac{u^2}{16\sigma^2\sum_{t=1}^n v_t}\Big).
$$
\end{lemma}

\section{Impact statement}\label{sec:impact}
This paper contributes foundational research in the areas of Transformer expressiveness power within the machine learning community. Our primary goal is to advance the understanding for Transformers.
Given the scope of this research, we do not anticipate immediate ethical concerns or direct societal consequences. Therefore, we believe there are no specific ethical considerations or immediate societal impacts to be emphasized in the context of this work.

\end{document}

%% file: myheaders.tex
\usepackage{amsfonts,amsmath,amssymb,amsthm,bm}
\usepackage{algorithm}
\usepackage{algorithmic}
\usepackage{mathrsfs}
\usepackage{tabularx}
\usepackage{multirow}
\usepackage{threeparttable}
\usepackage{makecell}
\usepackage{booktabs} 
\usepackage{appendix}
\usepackage{authblk}

\newtheorem{lemma}{Lemma}
\newtheorem{theorem}{Theorem}

\ifx\assumption\undefined
\newtheorem{assumption}{Assumption}
\fi

\newcommand{\cO}{{\mathcal{O}}}
\newcommand{\cH}{{\mathcal{H}}}

\newcommand{\cS}{{\mathcal{S}}}
\newcommand{\rR}{{\mathbb{R}}}
\newcommand{\E}{{\mathbb{E}}}
\newcommand{\Pb}{{\mathbb{P}}}

\newcommand{\md}{\mathrm{d}}

\newcommand{\od}{o_{\mathrm{delim}}}
\newcommand{\bbT}{\mathbb{T}}
\newcommand{\bbM}{\mathbb{M}}
\newcommand{\te}{\mathrm{test}}
\newcommand{\pe}{s}
\newcommand{\attn}{\mathcal{A}ttn}
\newcommand{\softmax}{\mathrm{Softmax}}
\newcommand{\relu}{\mathrm{ReLU}}
\newcommand{\vecc}{\mathcal{V}ec}

\renewcommand\bm[1]{{#1}}





%% file: main.bbl
\begin{thebibliography}{}

\bibitem[Achiam et~al., 2023]{achiam2023gpt}
Achiam, J., Adler, S., Agarwal, S., Ahmad, L., Akkaya, I., Aleman, F.~L., Almeida, D., Altenschmidt, J., Altman, S., Anadkat, S., et~al. (2023).
\newblock Gpt-4 technical report.
\newblock {\em arXiv preprint arXiv:2303.08774}.

\bibitem[Agarwal et~al., 2020]{agarwal2020flambe}
Agarwal, A., Kakade, S., Krishnamurthy, A., and Sun, W. (2020).
\newblock Flambe: Structural complexity and representation learning of low rank mdps.
\newblock {\em Advances in neural information processing systems}, 33:20095--20107.

\bibitem[Aky{\"u}rek et~al., 2022]{akyurek2022learning}
Aky{\"u}rek, E., Schuurmans, D., Andreas, J., Ma, T., and Zhou, D. (2022).
\newblock What learning algorithm is in-context learning? investigations with linear models.
\newblock {\em arXiv preprint arXiv:2211.15661}.

\bibitem[Bai et~al., 2023]{bai2023transformers}
Bai, Y., Chen, F., Wang, H., Xiong, C., and Mei, S. (2023).
\newblock Transformers as statisticians: Provable in-context learning with in-context algorithm selection.
\newblock {\em Advances in neural information processing systems}, 36:57125--57211.

\bibitem[Bai et~al., 2024]{bai2024transformers}
Bai, Y., Chen, F., Wang, H., Xiong, C., and Mei, S. (2024).
\newblock Transformers as statisticians: Provable in-context learning with in-context algorithm selection.
\newblock {\em Advances in neural information processing systems}, 36.

\bibitem[Baum and Eagon, 1967]{baum1967inequality}
Baum, L.~E. and Eagon, J.~A. (1967).
\newblock An inequality with applications to statistical estimation for probabilistic functions of markov processes and to a model for ecology.

\bibitem[Brown et~al., 2020]{brown2020language}
Brown, T., Mann, B., Ryder, N., Subbiah, M., Kaplan, J.~D., Dhariwal, P., Neelakantan, A., Shyam, P., Sastry, G., Askell, A., et~al. (2020).
\newblock Language models are few-shot learners.
\newblock {\em Advances in neural information processing systems}, 33:1877--1901.

\bibitem[Chiu et~al., 2021]{chiu2021low}
Chiu, J., Deng, Y., and Rush, A. (2021).
\newblock Low-rank constraints for fast inference in structured models.
\newblock {\em Advances in Neural Information Processing Systems}, 34:2887--2898.

\bibitem[Dai et~al., 2022]{dai2022can}
Dai, D., Sun, Y., Dong, L., Hao, Y., Ma, S., Sui, Z., and Wei, F. (2022).
\newblock Why can gpt learn in-context? language models implicitly perform gradient descent as meta-optimizers.
\newblock {\em arXiv preprint arXiv:2212.10559}.

\bibitem[Dedieu et~al., 2019]{dedieu2019learning}
Dedieu, A., Gothoskar, N., Swingle, S., Lehrach, W., L{\'a}zaro-Gredilla, M., and George, D. (2019).
\newblock Learning higher-order sequential structure with cloned hmms.
\newblock {\em arXiv preprint arXiv:1905.00507}.

\bibitem[Dong et~al., 2023]{dong2022survey}
Dong, Y. et~al. (2023).
\newblock A survey of in-context learning: Recent progress and future directions.
\newblock {\em arXiv preprint arXiv:2301.00234}.

\bibitem[Dubey et~al., 2024]{dubey2024llama}
Dubey, A., Jauhri, A., Pandey, A., Kadian, A., Al-Dahle, A., Letman, A., Mathur, A., Schelten, A., Yang, A., Fan, A., et~al. (2024).
\newblock The llama 3 herd of models.
\newblock {\em ArXiv preprint}, abs/2407.21783.

\bibitem[Fan et~al., 2023]{fan2023provably}
Fan, J., Wang, Z., Yang, Z., and Ye, C. (2023).
\newblock Provably efficient high-dimensional bandit learning with batched feedbacks.
\newblock {\em arXiv preprint arXiv:2311.13180}.

\bibitem[Felzenszwalb et~al., 2003]{felzenszwalb2003fast}
Felzenszwalb, P., Huttenlocher, D., and Kleinberg, J. (2003).
\newblock Fast algorithms for large-state-space hmms with applications to web usage analysis.
\newblock {\em Advances in neural information processing systems}, 16.

\bibitem[Garg et~al., 2022]{garg2022can}
Garg, S., Tsipras, D., Liang, P.~S., and Valiant, G. (2022).
\newblock What can transformers learn in-context? a case study of simple function classes.
\newblock {\em Advances in Neural Information Processing Systems}, 35:30583--30598.

\bibitem[Guo et~al., 2023a]{guo2023provably}
Guo, J., Li, Z., Wang, H., Wang, M., Yang, Z., and Zhang, X. (2023a).
\newblock Provably efficient representation learning with tractable planning in low-rank pomdp.
\newblock In {\em International Conference on Machine Learning}, pages 11967--11997. PMLR.

\bibitem[Guo et~al., 2023b]{guo2023transformers}
Guo, T., Hu, W., Mei, S., Wang, H., Xiong, C., Savarese, S., and Bai, Y. (2023b).
\newblock How do transformers learn in-context beyond simple functions? a case study on learning with representations.
\newblock {\em arXiv preprint arXiv:2310.10616}.

\bibitem[Hamilton et~al., 2013]{hamilton2013modelling}
Hamilton, W.~L., Fard, M.~M., and Pineau, J. (2013).
\newblock Modelling sparse dynamical systems with compressed predictive state representations.
\newblock In {\em International Conference on Machine Learning}, pages 178--186. PMLR.

\bibitem[Hsu et~al., 2012]{hsu2012spectral}
Hsu, D., Kakade, S.~M., and Zhang, T. (2012).
\newblock A spectral algorithm for learning hidden markov models.
\newblock {\em Journal of Computer and System Sciences}, 78(5):1460--1480.

\bibitem[Jiang, 2023]{jiang2023latent}
Jiang, H. (2023).
\newblock A latent space theory for emergent abilities in large language models.
\newblock {\em arXiv preprint arXiv:2304.09960}.

\bibitem[Kulesza et~al., 2015]{kulesza2015spectral}
Kulesza, A., Jiang, N., and Singh, S. (2015).
\newblock Spectral learning of predictive state representations with insufficient statistics.
\newblock In {\em Proceedings of the AAAI Conference on Artificial Intelligence}, volume~29.

\bibitem[Lin et~al., 2023]{lin2023transformers}
Lin, L., Bai, Y., and Mei, S. (2023).
\newblock Transformers as decision makers: Provable in-context reinforcement learning via supervised pretraining.
\newblock {\em arXiv preprint arXiv:2310.08566}.

\bibitem[Liu et~al., 2024]{liu2024deepseek}
Liu, A., Feng, B., Xue, B., Wang, B., Wu, B., Lu, C., Zhao, C., Deng, C., Zhang, C., Ruan, C., et~al. (2024).
\newblock Deepseek-v3 technical report.
\newblock {\em arXiv preprint arXiv:2412.19437}.

\bibitem[Liu et~al., 2022a]{liu2022transformers}
Liu, B., Ash, J.~T., Goel, S., Krishnamurthy, A., and Zhang, C. (2022a).
\newblock Transformers learn shortcuts to automata.
\newblock {\em arXiv preprint arXiv:2210.10749}.

\bibitem[Liu et~al., 2022b]{liu2022partially}
Liu, Q., Chung, A., Szepesv{\'a}ri, C., and Jin, C. (2022b).
\newblock When is partially observable reinforcement learning not scary?
\newblock In {\em Conference on Learning Theory}, pages 5175--5220. PMLR.

\bibitem[Mahankali et~al., 2023]{mahankali2023one}
Mahankali, A., Hashimoto, T.~B., and Ma, T. (2023).
\newblock One step of gradient descent is provably the optimal in-context learner with one layer of linear self-attention.
\newblock {\em arXiv preprint arXiv:2307.03576}.

\bibitem[Min et~al., 2022]{min2022rethinking}
Min, S. et~al. (2022).
\newblock Rethinking the role of demonstrations: What makes in-context learning work?
\newblock In {\em EMNLP}.

\bibitem[Nichani et~al., 2024]{nichani2024transformers}
Nichani, E., Damian, A., and Lee, J.~D. (2024).
\newblock How transformers learn causal structure with gradient descent.
\newblock {\em arXiv preprint arXiv:2402.14735}.

\bibitem[Rabiner, 1989]{rabiner1989tutorial}
Rabiner, L.~R. (1989).
\newblock A tutorial on hidden markov models and selected applications in speech recognition.
\newblock {\em Proceedings of the IEEE}, 77(2):257--286.

\bibitem[Sander et~al., 2024]{sander2024transformers}
Sander, M.~E., Giryes, R., Suzuki, T., Blondel, M., and Peyr{\'e}, G. (2024).
\newblock How do transformers perform in-context autoregressive learning?
\newblock {\em arXiv preprint arXiv:2402.05787}.

\bibitem[Siddiqi et~al., 2010]{siddiqi2010reduced}
Siddiqi, S., Boots, B., and Gordon, G. (2010).
\newblock Reduced-rank hidden markov models.
\newblock In {\em Proceedings of the Thirteenth International Conference on Artificial Intelligence and Statistics}, pages 741--748. JMLR Workshop and Conference Proceedings.

\bibitem[Siddiqi and Moore, 2005]{siddiqi2005fast}
Siddiqi, S.~M. and Moore, A.~W. (2005).
\newblock Fast inference and learning in large-state-space hmms.
\newblock In {\em Proceedings of the 22nd international conference on Machine learning}, pages 800--807.

\bibitem[Song et~al., 2010]{song2010hilbert}
Song, L., Boots, B., Siddiqi, S.~M., Gordon, G.~J., and Smola, A. (2010).
\newblock Hilbert space embeddings of hidden markov models.

\bibitem[Su et~al., 2024]{su2024roformer}
Su, J., Ahmed, M., Lu, Y., Pan, S., Bo, W., and Liu, Y. (2024).
\newblock Roformer: Enhanced transformer with rotary position embedding.
\newblock {\em Neurocomputing}, 568:127063.

\bibitem[Team et~al., 2023]{team2023gemini}
Team, G., Anil, R., Borgeaud, S., Wu, Y., Alayrac, J.-B., Yu, J., Soricut, R., Schalkwyk, J., Dai, A.~M., Hauth, A., et~al. (2023).
\newblock Gemini: a family of highly capable multimodal models.
\newblock {\em arXiv preprint arXiv:2312.11805}.

\bibitem[Touvron et~al., 2023]{touvron2023llama}
Touvron, H., Lavril, T., Izacard, G., Martinet, X., Lachaux, M.-A., Lacroix, T., Rozière, B., Goyal, N., Hambro, E., Azhar, F., Rodriguez, A., Joulin, A., Grave, E., and Lample, G. (2023).
\newblock Llama: Open and efficient foundation language models.
\newblock {\em arXiv Preprint}.

\bibitem[Uehara et~al., 2022]{uehara2022provably}
Uehara, M., Sekhari, A., Lee, J.~D., Kallus, N., and Sun, W. (2022).
\newblock Provably efficient reinforcement learning in partially observable dynamical systems.
\newblock {\em Advances in Neural Information Processing Systems}, 35:578--592.

\bibitem[Uehara et~al., 2021]{uehara2021representation}
Uehara, M., Zhang, X., and Sun, W. (2021).
\newblock Representation learning for online and offline rl in low-rank mdps.
\newblock {\em arXiv preprint arXiv:2110.04652}.

\bibitem[Van~Overschee and De~Moor, 1995]{van1995unifying}
Van~Overschee, P. and De~Moor, B. (1995).
\newblock A unifying theorem for three subspace system identification algorithms.
\newblock {\em Automatica}, 31(12):1853--1864.

\bibitem[Von~Oswald et~al., 2023]{von2023transformers}
Von~Oswald, J., Niklasson, E., Randazzo, E., Sacramento, J., Mordvintsev, A., Zhmoginov, A., and Vladymyrov, M. (2023).
\newblock Transformers learn in-context by gradient descent.
\newblock In {\em International Conference on Machine Learning}, pages 35151--35174. PMLR.

\bibitem[Wang et~al., 2022]{wang2022embed}
Wang, L., Cai, Q., Yang, Z., and Wang, Z. (2022).
\newblock Embed to control partially observed systems: Representation learning with provable sample efficiency.
\newblock {\em arXiv preprint arXiv:2205.13476}.

\bibitem[Wang et~al., 2023]{wang2023large}
Wang, X., Zhu, W., and Wang, W.~Y. (2023).
\newblock Large language models are implicitly topic models: Explaining and finding good demonstrations for in-context learning.
\newblock {\em arXiv preprint arXiv:2301.11916}, page~3.

\bibitem[Wei et~al., 2022]{wei2022emergent}
Wei, J. et~al. (2022).
\newblock Emergent abilities of large language models.
\newblock {\em arXiv preprint arXiv:2206.07682}.

\bibitem[Wu et~al., 2025]{wu2025transformers}
Wu, D., He, Y., Cao, Y., Fan, J., and Liu, H. (2025).
\newblock Transformers and their roles as time series foundation models.
\newblock {\em arXiv preprint arXiv:2502.03383}.

\bibitem[Xie et~al., 2021]{xie2021explanation}
Xie, S.~M., Raghunathan, A., Liang, P., and Ma, T. (2021).
\newblock An explanation of in-context learning as implicit bayesian inference.
\newblock {\em arXiv preprint arXiv:2111.02080}.

\bibitem[Ye et~al., 2023]{ye2023corruption}
Ye, C., Xiong, W., Gu, Q., and Zhang, T. (2023).
\newblock Corruption-robust algorithms with uncertainty weighting for nonlinear contextual bandits and markov decision processes.
\newblock In {\em International Conference on Machine Learning}, pages 39834--39863. PMLR.

\bibitem[Ye et~al., 2024]{ye2024physics}
Ye, T., Xu, Z., Li, Y., and Allen-Zhu, Z. (2024).
\newblock Physics of language models: Part 2.1, grade-school math and the hidden reasoning process.
\newblock In {\em The Thirteenth International Conference on Learning Representations}.

\bibitem[Zhan et~al., 2022]{zhan2022pac}
Zhan, W., Uehara, M., Sun, W., and Lee, J.~D. (2022).
\newblock Pac reinforcement learning for predictive state representations.
\newblock {\em arXiv preprint arXiv:2207.05738}.

\bibitem[Zhang, 2023]{zhang_2023_ltbook}
Zhang, T. (2023).
\newblock {\em Mathematical Analysis of Machine Learning Algorithms}.
\newblock Cambridge University Press.

\bibitem[Zhong et~al., 2022]{zhong2022gec}
Zhong, H., Xiong, W., Zheng, S., Wang, L., Wang, Z., Yang, Z., and Zhang, T. (2022).
\newblock Gec: A unified framework for interactive decision making in mdp, pomdp, and beyond.
\newblock {\em arXiv preprint arXiv:2211.01962}.

\end{thebibliography}
